\documentclass[journal]{IEEEtran}
\usepackage{amsfonts,epsfig,array,multirow,graphicx,amsmath,ltablex,
tabularx,setspace,amssymb,multirow}
\usepackage{schemabloc,tikz}
\usepackage{graphicx}
\usepackage{amsmath,bm}
\usepackage{amsthm}
\usepackage{multirow}
\usepackage{subcaption}
\newtheorem{theorem}{Theorem}
\usepackage{makecell}
\usepackage{relsize}
\usepackage{caption}
\usepackage{subfig}
\def\textsub#1{\ensuremath{_{\mbox{\textscale{.6}{#1}}}}}

\ifCLASSINFOpdf
\else
\fi
\begin{document}
%
\title{Image Reconstruction using Matched Wavelet Estimated from Data Sensed Compressively  using Partial Canonical Identity Matrix}
%
%
%

\author{Naushad~Ansari,~\IEEEmembership{Student Member,~IEEE,}
        Anubha~Gupta,~\IEEEmembership{Senior Member,~IEEE,}
        \vspace{-1.2em}
\thanks{\newline Naushad Ansari and Anubha Gupta are with SBILab, Deptt. of ECE, IIIT-Delhi, India. emails: naushada@iiitd.ac.in and anubha@iiitd.ac.in.\newline Naushad Ansari has been supported by CSIR (Council of Scientific and Industrial Research), Govt. of India, for this research work.}
}

\maketitle
\begin{abstract}
This paper proposes a joint framework wherein lifting-based, separable, image-matched wavelets are estimated from compressively sensed (CS) images and used for the reconstruction of the same. Matched wavelet can be easily designed if full image is available. Also matched wavelet may provide better reconstruction results in CS application compared to standard wavelet sparsifying basis. Since in CS application, we have compressively sensed image instead of full image, existing methods of designing matched wavelet cannot be used. Thus, we propose a joint framework that estimates matched wavelet from the compressively sensed images and also reconstructs full images. This paper has three significant contributions. First, lifting-based, image-matched separable wavelet is designed from compressively sensed images and is also used to reconstruct the same. Second, a simple sensing matrix is employed to sample data at sub-Nyquist rate such that sensing and reconstruction time is reduced considerably without any noticeable degradation in the reconstruction performance. Third, a new multi-level L-Pyramid wavelet decomposition strategy is provided for separable wavelet implementation on images that leads to improved reconstruction performance. Compared to CS-based reconstruction using standard wavelets with Gaussian sensing matrix and with existing wavelet decomposition strategy, the proposed methodology provides faster and better image reconstruction in compressive sensing application.
\end{abstract}

\IEEEpeerreviewmaketitle
\vspace{-0.8em}
\section{Introduction}
\label{Intro}
\IEEEPARstart{C}{lassical} signal acquisition technique in signal processing branch involves sensing the full signal at or above the Nyquist rate. In general, this signal is transformed to a domain where it is compressible. Only some of the largest coefficients of transformed signal having sufficient amount of energy are stored and transmitted to the receiver end along with the position information of the transmitted coefficients. The transmitted signal is decoded at the receiver end to recover the original signal. Thus, this process involves sensing the full signal, although most of the samples in the transformed domain are to be discarded. 

In \cite{candes2006robust,donoho2006compressed}, researchers proposed compressed sensing (CS) method that combines sensing and compression in one stage, where instead of sampling the signal sample-wise above Nyquist rate, it's projections are captured via the measurement basis. These samples are very few as compared to those sampled at Nyquist rate. If the signal is sparse in some transform domain and also if the measurement basis are incoherent with the sparsifying basis, then the whole signal can be reconstructed with a good probability from a very few projections of the signal.

In this context, wavelets are extensively used as sparsifying basis in compressed sensing problem \cite{donoho2006compressed}. One of the advantage with wavelets is that there is no unique basis unlike Fourier transform. One may choose the set of basis depending upon the type of application. Since the wavelet basis are not unique, it is better to design wavelet that are matched to a given signal in a particular application. The designed wavelet basis are called signal-matched wavelet \cite{gupta2005new1, gupta2005new2, ansari2015signal}. 

Motivated with the above discussion, this paper proposes to design matched wavelets for compressive sensing application. Unlike previous works on matched wavelets \cite{gupta2005new1, gupta2005new2,ansari2015signal,chapa2000algorithms, claypoole1998adaptive,piella2005gradient,heijmans2002building} where wavelets are designed from fully sampled signal, this paper proposes a novel method of designing signal matched wavelets from compressively sensed images at the receiver. The proposed method employs lifting based framework, also called second generation wavelet framework, to design image-matched wavelet \cite{sweldens1996lifting}. 

Many researchers have designed wavelets using lifting \cite{dong2008signal,piella2005gradient,quellec2010adaptive,liu2008weighted,zhang2006design,blackburn2009two,vrankic2010adaptive}\cite{ansari2015signal,ansari2015Joint} that requires design of predict and update stage filters. For images-matched wavelet design, these filters are required to be designed from the given image itself. Although it is easy to design predict stage filters, design of update stage filters offer a real challenge. One of the criteria used in the literature to design the update stage filter is to minimize the reconstruction error\cite{dong2008signal}. In \cite{piella2005gradient}, update step is chosen based on the local gradient of the signal. In \cite{quellec2010adaptive}, authors use non-separable lifting based approach of wavelet implementation on images with regularity conditions imposed. In \cite{liu2008weighted}, directional interpolation in update stage is used wherein coefficients of interpolation filter adapt to the statistical property of image. In \cite{zhang2006design}, allpass filter in the lifting stage is used to design orthogonal IIR (Infinite Impulse Response) filterbank. In \cite{blackburn2009two}, geometry of the image is used to design wavelet via lifting that leads to local and anisotropic filters. Authors have designed nonseparable filterbanks that are pixel-wise adaptive to the local image feature in \cite{vrankic2010adaptive}. In \cite{ansari2015signal}, authors designed predict stage by minimizing the energy in the wavelet subspace domain and designed update stage by minimizing the difference between reconstructed and original signal leaving wavelet signal. 

However, all the methods discussed above design signal-matched wavelet using the full signal at the input. So far, to the best of our knowledge, no method has been proposed that designs signal or image-matched wavelets from compressive or partial measurements. This paper addresses this problem from the point of view of image reconstruction. Below are the salient contributions of this paper:
\begin{enumerate}
\item We propose design of image-matched separable wavelet in the lifting framework from compressively sensed image, that is also used to reconstruct the image.  
\item In general, Gaussian or Bernoulli measurement matrices are used in compressive sensing application \cite{donoho2006compressed},\cite{takhar2006new}. We propose to use partial canonical identity matrix to sample data at sub-Nyquist rate such that sensing time is reduced considerably.
\item For the separable 2D wavelet transform, a new multi-level wavelet decomposition strategy is proposed that leads to improved reconstruction performance. We name this new wavelet decomposition strategy as multi-level L-Pyramid wavelet decomposition.
\end{enumerate}

This paper is organized as follows: In section-\ref{Section For Background}, we briefly present the theory on compressive sensing and lifting-based wavelet design. In section-\ref{Section For CI}, we propose to use partial canonical identity (PCI) matrix to sub-sample images at sub-Nyquist rate. We show results on time complexity and performance of the conventionally used sensing matrices and PCI matrix to establish the use of latter in this work. In section-\ref{Section For Proposed Strategy}, we propose a new multi-level wavelet decomposition strategy for separable wavelet on images and compare the performance with the existing strategy. Section-\ref{Section For Proposed Method of Matched Wavelet} presents the proposed joint framework of matched wavelet estimation and image reconstruction in CS. We present experimental results in section-\ref{Section for Experimental Results} and conclusions in section-\ref{Section for Conclusion}. 
\vspace{-0.4em}
\section{Background}
\label{Section For Background}
In this section, we briefly present the theory of compressive sensing and lifting framework of wavelets for the sake of self-completeness of the paper.
\vspace{-1em}
\subsection{Compressed Sensing}
\label{Section for CS Theory}
Classical compression method involves two steps: sensing and compression wherein, first, an analog data is sampled at or above the Nyquist-rate and then, it is compressed by an appropriate transform coding process. In general, natural signals are sparse or compressible in some transform domain. For examples, if a signal is smooth, it is compressible in Fourier domain and if it is piece-wise smooth, it is sparse in the wavelet domain. To understand this process, let us consider a signal $\mathbf{x}$ of dimension $N \times 1$ that has been sensed by a traditional sensing technique at or above the Nyquist rate. This signal is next transformed to sparse domain with the help of a sparsifying basis $\psi_i, i=1,2,...,N $ and represented as:
\begin{equation}
\mathbf{x}=\mathbf{\Psi} \mathbf{s}
\end{equation}
A signal $\textbf{s}$ is $K-$ sparse if all but $K$ elements are zero, whereas a signal is compressible if its sorted coefficients obey the power low decay \cite{baraniuk2011introduction}
\begin{equation}
s_j=Cj^{-q}, \ j=1,2,...,N,
\end{equation} 
where $s_j$ represent the sorted coefficients and $q$ represents decay power parameter. For large value of $q$, decay of coefficients is faster and correspondingly, signal is more compressible. In compression, some of the largest coefficients of the transformed signal are kept and all other coefficients are discarded. These coefficients along with their location information are sent to the receiver. Having the knowledge of the sparsifying basis and signal coefficients along with their positions in the original signal, signal is reconstructed back at the receiver end. 

The above process consisting of first sensing the whole signal and then discarding many of its transform domain coefficients is inefficient. Compressive sampling or sensing \cite{candes2006robust,donoho2006compressed,candes2006near} combines these two processes. Instead of sampling the signal at or above the Nyquist rate, signal's linear  projection on some measurement basis $\phi_i$ are obtained. If $\phi_i$ is the $i^{th}$ measurement basis, then $i^{th}$ observation of the projected signal is given by:
\vspace{-0.5em}
\begin{equation}
y[i]=\sum_{j=0}^{N-1} \phi_{i,j} x[j], i=0,1,...,M-1,
\vspace{-0.5em}
\end{equation}
where $M$ is the number of linear projections of the signal. In compact form, this can be written as: 
\begin{align}
\mathbf{y}_{M \times 1}=&\mathbf{\Phi}_{M \times N} \mathbf{x}_{N \times 1}, \nonumber \\
=&\mathbf{\Phi \Psi s} \nonumber \\
=&\mathbf{As},
\label{CS Equation}
\end{align}
where $i^{th}$ measurement basis is stacked as a row of the matrix $\mathbf{\Phi}$ and $ \mathbf{A}=\mathbf{\Phi \Psi}$. CS theory states that the original signal of length $N$ can be recovered with very high probability, if number of linear projections $M$ are taken such that \cite{candes2006robust}:
\begin{equation}
M \geq CK \log(N/K)
\label{CS Condition}
\end{equation} 
where $K$ is the sparsity of the signal, \textit{C} is some constant, and $M \ll N$ in general.

Equation \eqref{CS Equation} represents under-determined system of linear equations with $\mathbf{y}=\bm{\Phi}\mathbf{\hat{x}}$ having infinite many solutions $\mathbf{\hat{x}}$. But if the signal is sparse in some transform domain $\mathbf{\Psi}$, \eqref{CS Equation} can be solved using 
$l_0$ minimization as below:
\begin{equation}
\tilde{\mathbf{s}} = \min_{\mathbf{s}} ||\mathbf{s}||_0 \text{     subject to:     } \mathbf{y}=\mathbf{A} \mathbf{s}.
\end{equation} 
The above problem is NP-hard to solve. It has been shown in \cite{sharon2007computation} that $l_1$ minimization
 \begin{equation}
\tilde{\mathbf{s}} = \min_{\mathbf{s}} ||\mathbf{s}||_1 \text{     subject to:     } \mathbf{y}=\mathbf{A} \mathbf{s},
\label{CS solution}
\end{equation} 
provides the same solution as $l_0$ minimization. Here $||\mathbf{v}||_1$ denotes the $l^1$ norm or sum of the absolute values of the vector $\mathbf{v}$. $l^1$ minimization is known as Basis Pursuit (BP) in literature and can be solved by linear programming \cite{chen2001atomic}. 

Compressive sensing is being used increasingly in image reconstruction, also called as compressive imaging (CI). For example, let us consider an image $\mathbf{X}$ of dimension $m \times n$, that is compressively sensed by a measurement matrix $\mathbf{\Phi}$. These measurements are given by
\begin{align}
\mathbf{y}&=\mathbf{\Phi} vec(\mathbf{X}),\nonumber \\
&=\mathbf{\Phi}\mathbf{x}
\label{eq5}
\end{align}
where $vec(\mathbf{X})=\textbf{x}$ denotes the vector of length $N=mn$ of image $\mathbf{X}$ and the measured signal $\mathbf{y}$ is of dimension $M \times 1$, where $M$ is the number of compressive measurements. It has been observed that natural images, in general, are compressible in DCT (discrete cosine transform) \cite{shen1998dct} and wavelet domains \cite{skodras2001jpeg}. Hence, DCT or wavelet can be applied as separable transforms on images and used as sparsifying basis $\mathbf{\Psi}$ in \eqref{CS Equation} in CS-based image reconstruction. 

The measurement or the sensing basis $\mathbf{\Phi}$ can be chosen such that it satisfies Restricted Isometry Property (RIP) \cite{candes2008restricted} and coherency property \cite{donoho2006compressed}. Some of the examples of the measurement matrices that satisfy these properties are random matrices with entries taken from i.i.d. Gaussian distribution \cite{chen2005condition}, random matrices with entries taken from uniform Bernoulli distributions \cite{candes2006near}, and Fourier matrix \cite{donoho2006compressed}; although several other structured measurement matrices such as toeplitz and circulant matrices are also being used \cite{bajwa2007toeplitz,rauhut2010compressive,rauhut2012restricted}.
\vspace{-1em}
\subsection{Lifting Theory}
\label{Section For Lifting Theory}
Lifting is a technique for either factoring existing wavelet filters into a finite sequence of smaller filtering steps or constructing new customized wavelets from existing wavelets \cite{daubechies1998factoring}. Lifting-based wavelet design is also known as second generation wavelet design. The design is modular, guarantees perfect reconstruction at every stage, and supports non-linear filters. A general lifting scheme consists of three steps: Split, Predict, and Update (Refer to Fig. 1). 

\textit{Split}: In the split step, input signal is split into two disjoint sets of samples, generally, even and odd indexed samples, labeled as $x_e[n]$ and $x_o[n]$, respectively. The original signal can be recovered perfectly by interlacing or combining the two sample streams. The corresponding filterbank is called as the \textit{Lazy Wavelet} system \cite{sweldens1996lifting} and is similar to the structure shown in Fig. 2 with analysis filters labeled as $H_0(z)=Z\{h_0[n]\}$, $H_1(z)=Z\{h_1[n]\}$ and synthesis filters as $F_0(z)=Z\{f_0[n]\}$, $F_1(z)=Z\{f_1[n]\}$.    

\textit{Predict Step}: In the predict step, also known as dual Lifting step, one of the two disjoint set of samples is predicted from the other set of samples. For example, in Fig. 1(a), we predict odd set of samples from the neighboring even samples by using the predictor $P \equiv T(z)$. Predict step is equivalent to applying a highpass filter on the input signal. Predict step modifies the highpass filter of the analysis end and lowpass filter of the synthesis end, without altering other filters, according to the following relations: 
\begin{equation}
H_{1}^{new}(z) = H_1(z)-H_0(z)T(z^{2}),
 \label{eq:no1}
\end{equation}
\begin{equation}
F_{0}^{new}(z) = F_0(z)+F_1(z)T(z^{2}).
 \label{eq:no2}
\end{equation}

\textit{Update Step}: In the update step, also known as primal lifting step, predicting samples of the predict step are updated with the predicted samples to provide the approximate coefficients of the signal. The signal is updated with $U \equiv S(z)$ (refer to Fig. 1). This step modifies the analysis lowpass filter and synthesis highpass filter according to the following relation:
\begin{equation}
H_{0}^{new}(z) = H_0(z)+H_1(z)S(z^{2}),
 \label{eq:no3}
\end{equation}
\begin{equation}
F_{1}^{new}(z) = F_1(z)-F_0(z)S(z^{2}).
 \label{eq:no4}
\end{equation}

Once all the filters are designed, Fig. 1 can be equivalently drawn as Fig. \ref{WaveletFilterbank} or any existing wavelet system of Fig. \ref{WaveletFilterbank} can be equivalently broken into lifting steps of Fig. 1. One of the major advantages of lifting scheme is that each stage (predict or update) is invertible. Hence, perfect reconstruction (PR) is guaranteed after every step. 
\begin{figure}
\begin{center}
\includegraphics[scale=0.6, trim =6mm 4mm 6mm 1mm]{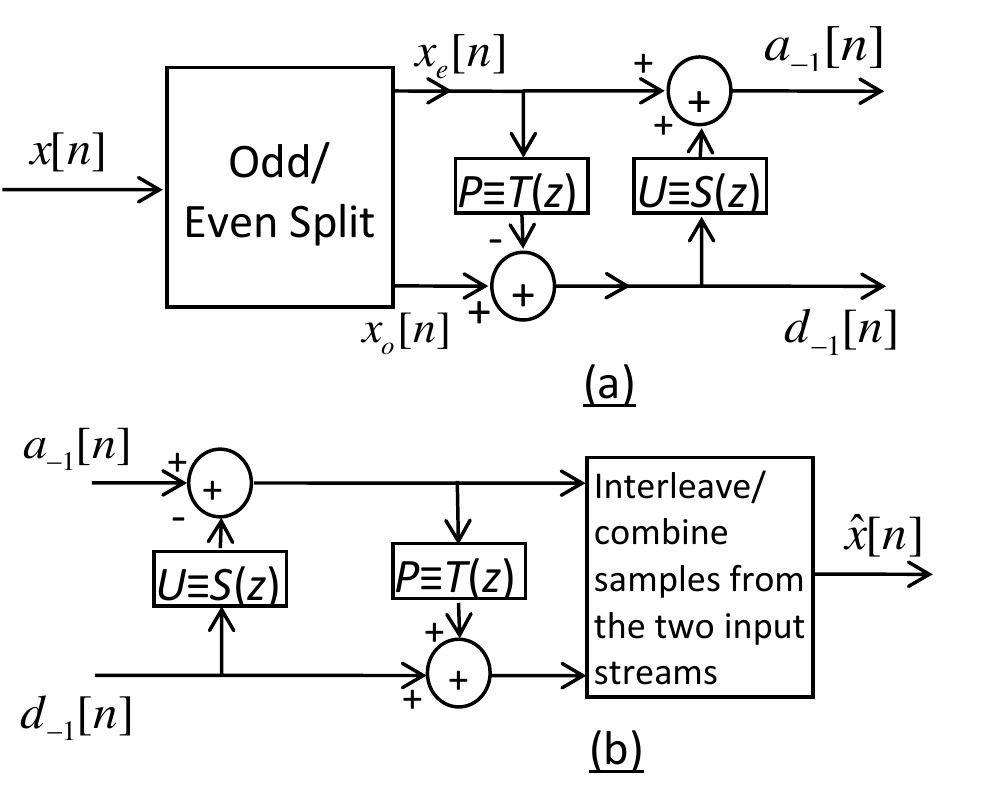}
\end{center}
\vspace{-0.4em}
\caption{Steps of Lifting: Split, Predict and Update}
\vspace{-1.4em}
\end{figure}
\begin{figure}
\begin{center}
\includegraphics[scale=0.6, trim =6mm 3mm 6mm 1mm]{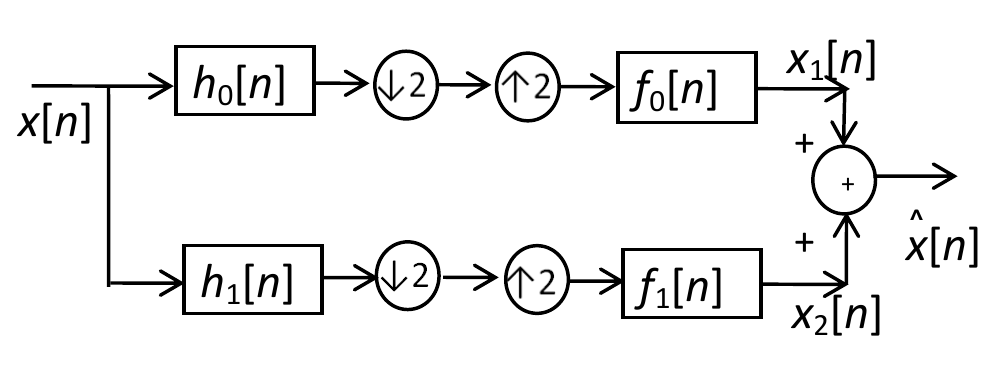}
\vspace{-0.6em}
\caption{Two Channel Biorthogonal Wavelet System}
\label{WaveletFilterbank}
\end{center}
\vspace{-2.0em}
\end{figure}
\vspace{-0.8em}
\section{Compressive Sensing of Images using Proposed Sensing Matrix}
\label{Section For CI}
In this section, first we establish the need for a different, other than conventional, sensing matrix. Next, we discuss the proposed matrix. And later, we present results to show the comparison of time complexity and reconstruction performance with the proposed matrix in CS based image reconstruction. 
\vspace{-1em}
\subsection{Context}
In today's world, size of images are increasingly large and $N$ generally approaches to millions of samples. This large size imaging poses challenges for CS-based image reconstruction. First challenge is the huge size of measurement matrix $\Phi$ that poses problems with storage and computation. Other problems include design of imaging system with larger space bandwidth product (SBP) and difficult calibration requirements \cite{rivenson2009compressed}.

In an attempt to overcome the above challenges, single pixel camera hardware architecture has been proposed in \cite{takhar2006new}. It replaces the traditional camera architecture and captures the inner product between the scene under view and measurement basis. Thus, the camera captures one pixel at a time that is a linear combination of all pixel samples of the image. This process is repeated $M$ number of times with $M \ll N$. These are called the compressive measurements and are transmitted to the receiver where full sized image is reconstructed by employing the theory of CS-based reconstruction. For more information on single pixel camera, reader may refer to \cite{takhar2006new}. 

With the above explained architecture, a single pixel camera replaces the photon detector array of a traditional camera by a single photon detector; thus, reducing the size, cost, and complexity of the imaging architecture. 

Although the above hardware architecture implements compressive imaging (CI) nicely, it suffers with some difficulties including sensor dynamic range, A/D quantization, and photon counting effects \cite{takhar2006new}. Also, this process is time consuming as one has to wait for $M$ samples that are captured sequentially. This is a serious problem in real-time applications, say, when one has to record a video using camera as the scene may change while capturing samples sequentially of the current scene under view. Also, since $M$ linear projections are captured instead of $N$ pixel samples, it ``effectively" samples the image at Nyquist rate instead of actually sampling the image at the sub-Nyquist rate. 
\vspace{-0.8em}
\subsection{Proposed Use of Partial Canonical Identity Sensing Matrix}
We propose to use PCI sensing matrix that as per our knowledge, is the simplest sensing matrix proposed so far and ``actually" senses the image at sub-Nyquist rate by capturing less number of pixels without sensing information about every pixel. This is explained as below.

Consider an image $\mathbf{X}$ of dimension $m \times n$. Instead of sampling all the $N (N=m n)$ pixels of the image using the sensor array of the traditional camera, we capture $M$ samples of the image using the proposed measurement matrix $\mathbf{\Phi}^p$, where $M \ll N$. The measurement matrix $\mathbf{\Phi}^p$ has the entries shown below:
\begin{equation}
\mathbf{\Phi}^p_{i,j}=
\left\{
		\begin{array}{ll}
			1 \quad \mbox{  if  } i \in \{1,2,...,M\} \mbox{ and } j \in \Omega \\
			0 \quad \mbox{otherwise}		
		\end{array}    
\right.
\end{equation}
where $\Omega\subset \{1,2,...,N\} $ such that $|\Omega|=M $ and $|.|$ denotes the cardinality of the the set. We name the proposed sensing matrix as partial canonical identity (PCI) matrix because it consists of partially selected and permuted rows of the identity matrix.  

The PCI matrix captures only $M$ samples of the actual image; thus, actually sub-samples the original image. This can be accomplished by using existing cameras by switching ON only $M$ sensors of the sensor array. This is unlike the single pixel camera where every captured pixel is the linear combination of the entire image pixel set. Also, in single pixel camera, one has to wait for $M$ units of time to sense $M$ number of samples, whereas all $M$ samples are sensed in one unit of time in the case of PCI matrix. Thus, PCI matrix reduces the sensing time by a factor of $M$ in comparison to a single pixel camera.
\vspace{-0.5em}
\subsection{Results using PCI Sensing Matrix}
Consider a sub-sampled Lena image, shown in Fig. \ref{SubSampled Image}, with only 50\% samples captured via PCI sensing matrix. The un-captured positions are filled with zeros. Fig. \ref{Reconstructed Image} shows the image reconstructed from this sub-sampled image using \eqref{CS solution} with standard wavelet `db4' as the sparsifying basis. From Fig. \ref{Reconstructed Image}, we note that PCI matrix has good ability of reconstructing images from compressively sensed data and can be used in CS-based image reconstruction. Fig.  \ref{TimeComparisonWithMeasurementMatrices} and \ref{PSNRComparisonWithMeasurementMatrices} provide detailed results.
\begin{figure}[!ht]
\centering
\begin{subfigure}[b]{0.24\textwidth}
\centering
\includegraphics[scale=0.48]{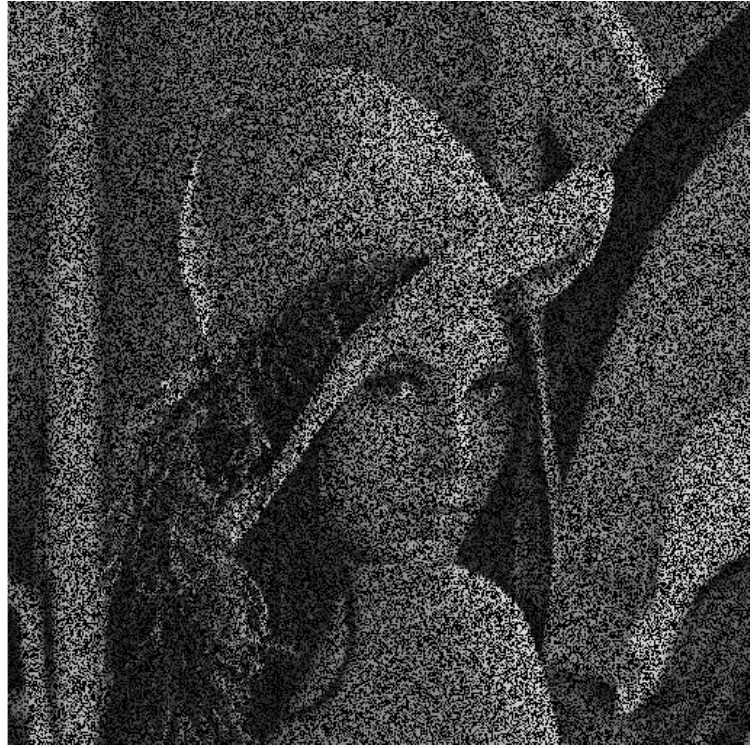}
\caption{}
\label{SubSampled Image}
\end{subfigure}
\begin{subfigure}[b]{0.24\textwidth}
\centering
\includegraphics[scale=0.46]{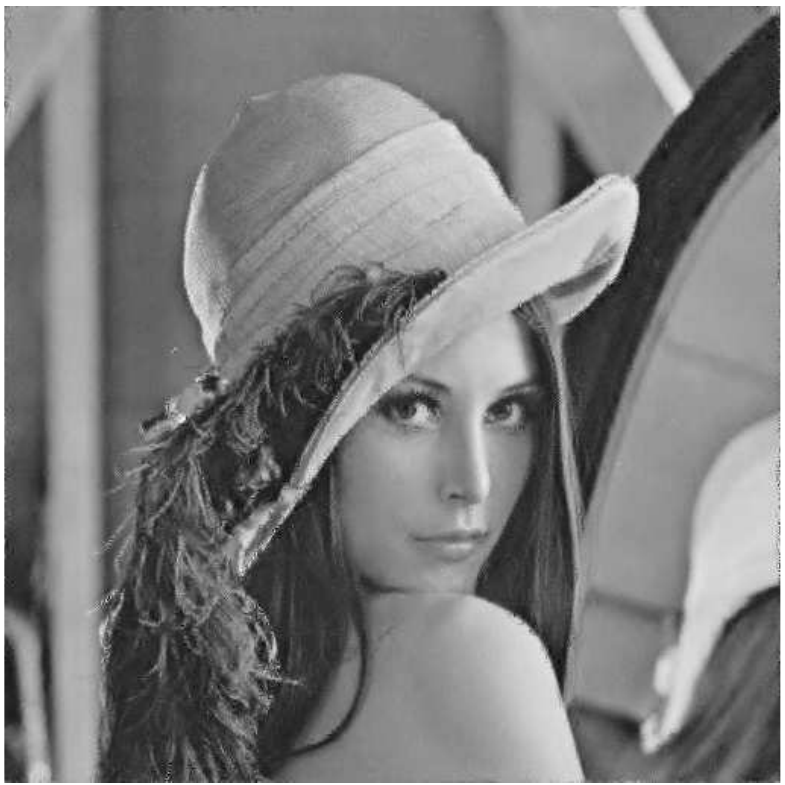}
\caption{}
\label{Reconstructed Image}
\end{subfigure}
\vspace{-1.7em}
\captionsetup{justification=centering}
\caption{\small (a) Image captured using PCI matrix with 50\% sampling ratio and with zeros filled at positions not sampled (b) Image reconstructed using `db4' wavelet as the sparsifying basis in \eqref{CS solution}.}
\vspace{-1.8em}
\end{figure}
\begin{figure*}[!ht]
\centering
\begin{subfigure}[b]{0.32\textwidth}
\centering
\includegraphics[scale=0.51]{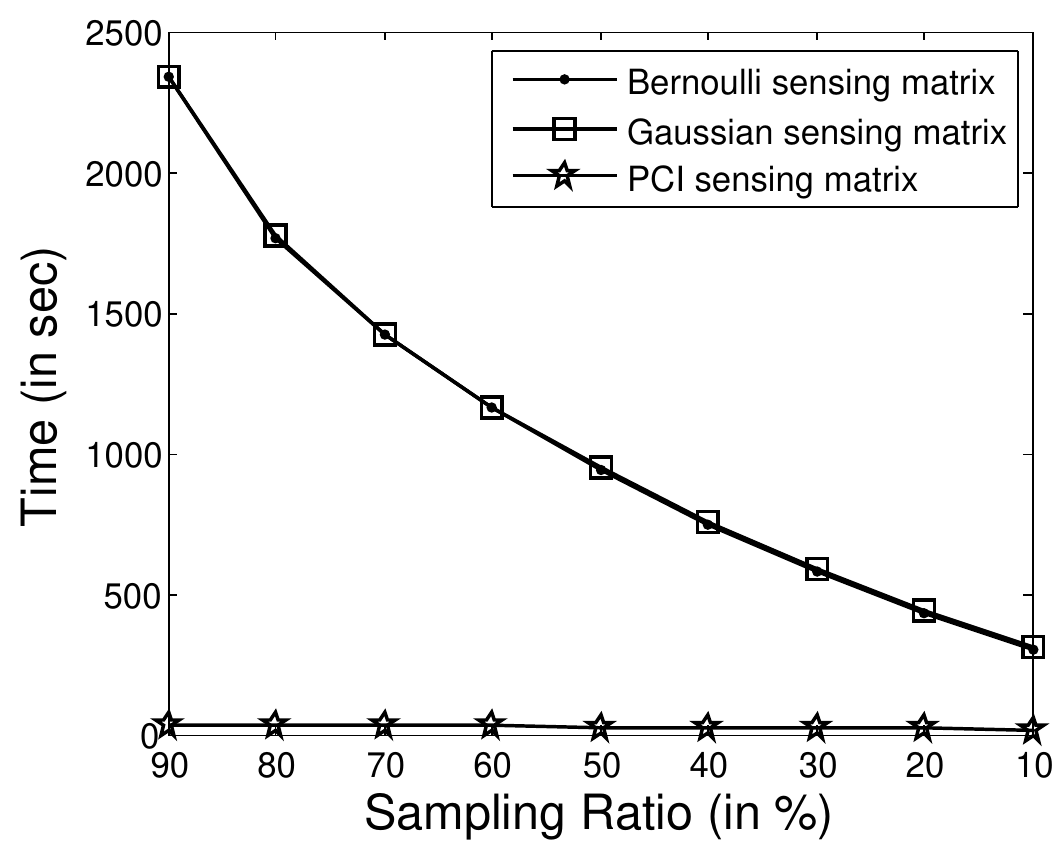}
\caption{}
\end{subfigure}
\begin{subfigure}[b]{0.32\textwidth}
\centering
\includegraphics[scale=0.51]{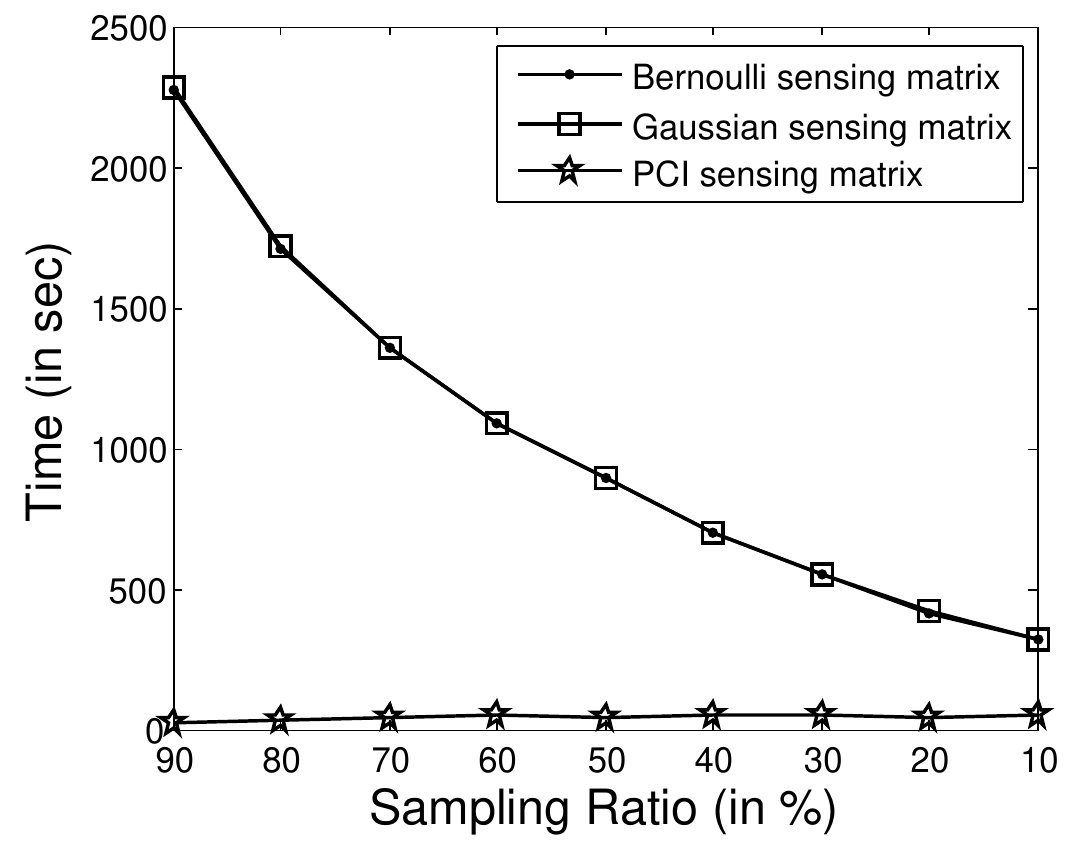}
\caption{}
\end{subfigure}
\begin{subfigure}[b]{0.32\textwidth}
\centering
\includegraphics[scale=0.51]{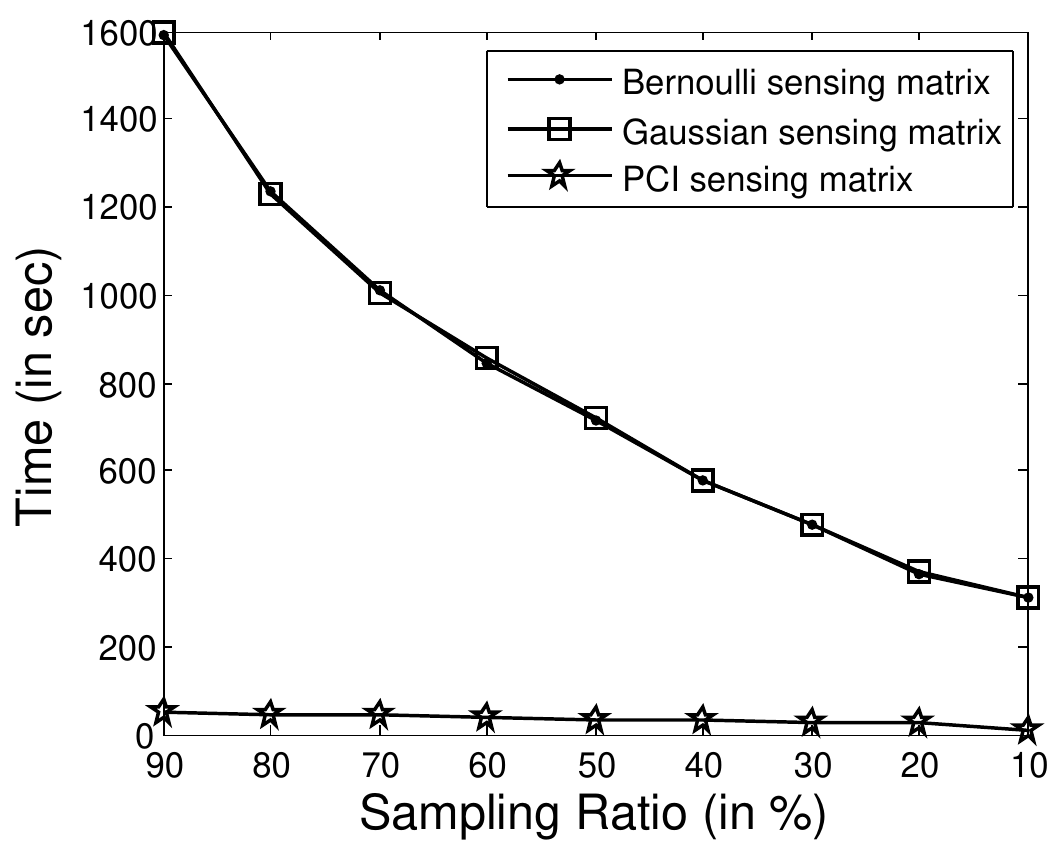}
\caption{}
\end{subfigure}
\vspace{-0.7em}
\captionsetup{justification=centering}
\caption{\small Time comparison with different measurement matrices. (a) Image `Beads' (b) Image `Lena' (c) Image `House'}
\vspace{-1.5em}
\label{TimeComparisonWithMeasurementMatrices}
\end{figure*}
\begin{figure*}[t]
\centering
\begin{subfigure}[b]{0.32\textwidth}
\centering
\includegraphics[scale=0.51]{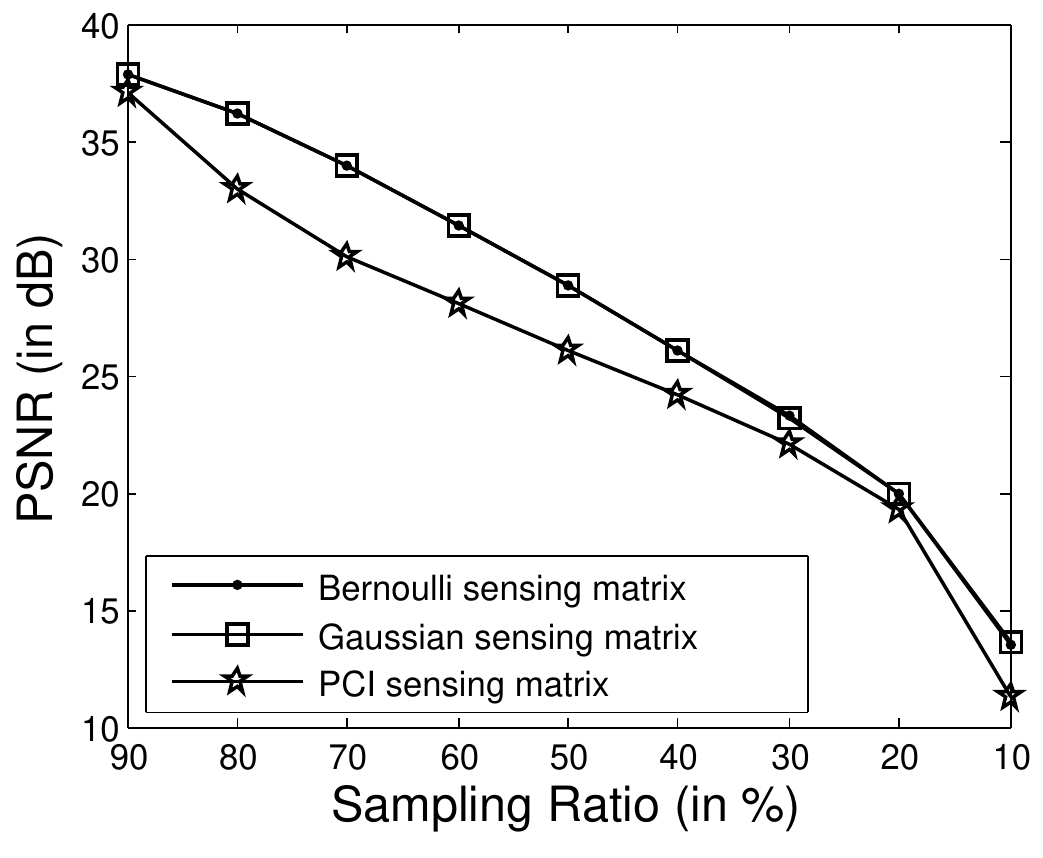}
\caption{}
\vspace{-0.3em}
\end{subfigure}
\begin{subfigure}[b]{0.32\textwidth}
\centering
\includegraphics[scale=0.51]{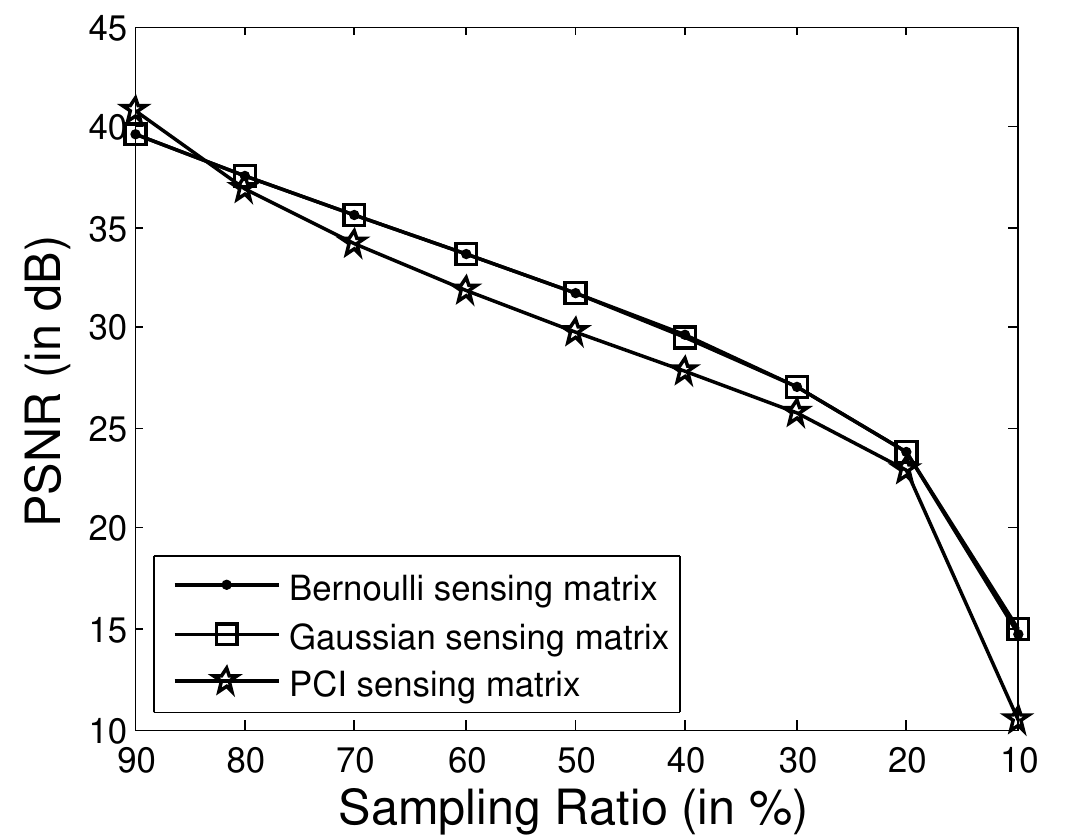}
\caption{}
\end{subfigure}
\begin{subfigure}[b]{0.32\textwidth}
\centering
\includegraphics[scale=0.51]{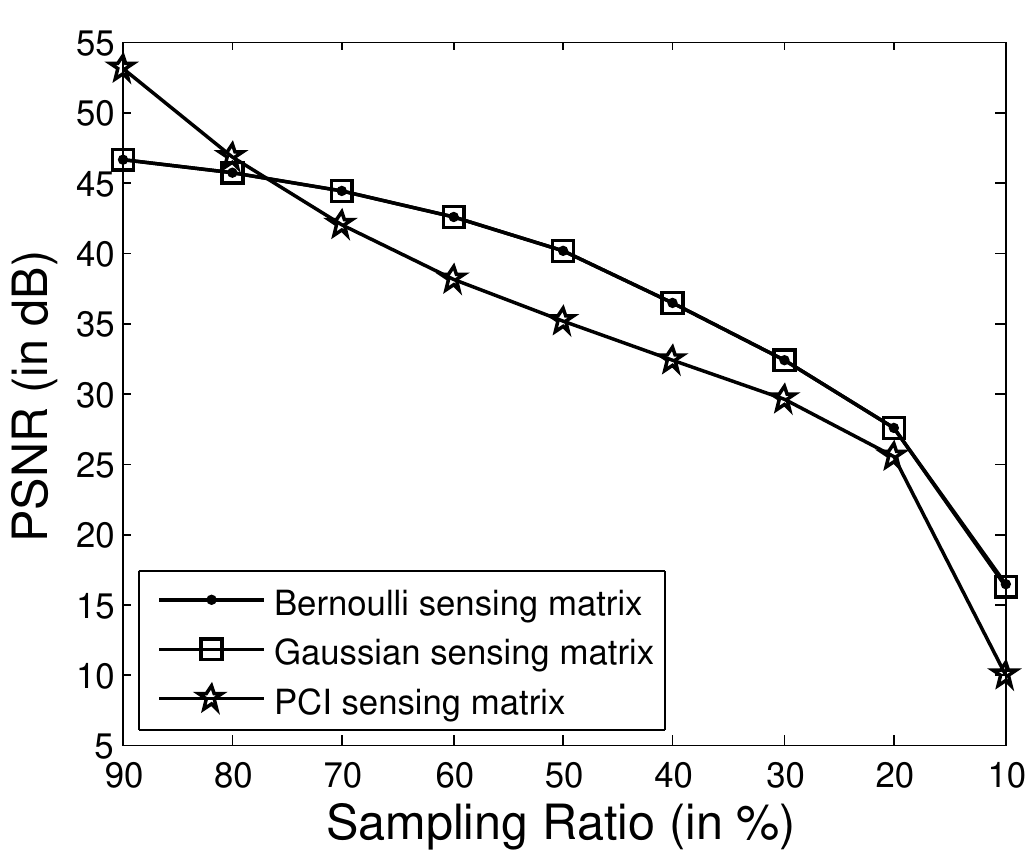}
\caption{}
\end{subfigure}
\vspace{-0.7em}
\captionsetup{justification=centering}
\caption{\small Reconstruction accuracy in terms of PSNR (in dB) with different measurement matrices. (a) Image `Beads' (b) Image `Lena' (c) Image `House'}
\label{PSNRComparisonWithMeasurementMatrices}
\vspace{-2em}
\end{figure*}

Fig. \ref{TimeComparisonWithMeasurementMatrices} compares the time taken in image reconstruction from the measured samples with sampling ratios varying from 10\% to 90\% using different measurement matrices. We compare the reconstruction time taken using the PCI matrix, random Gaussian matrix, and Bernoulli random matrix with $\pm 1$ as its entries. Random Gaussian matrix is preferred in wide range of applications because of the ease in theoretical analysis, while the Bernoulli matrix depicts physical implementation of single pixel camera. Image is reconstructed using equation \eqref{CS solution}, where we have used standard Daubechies orthogonal wavelet `db4' as the sparsifying basis. We have used MATLAB solver spgl1 \cite{BergFriedlander:2008}, \cite{spgl1:2007} to solve \eqref{CS solution} that implements Basis Pursuit (BP) \cite{chen2001atomic}. 

Compressed sensing based reconstruction with Gaussian and Bernoulli sensing matrices is implemented using block compressed sensing \cite{gan2007block}. This is to note that reconstruction with PCI matrix requires only the position information of the sampled pixels instead of the information of all entries of $M \times N$ sensing matrix that simplifies reconstruction with PCI matrix. We compare reconstruction results on three images: `Beads', `Lena', and `House' as shown in Fig. \ref{Beads}, \ref{Lena} and \ref{House}, respectively. Each image is of size $512 \times 512$. We have chosen these images because they exhibit different spectral properties. For example, `Beads' is rich in high frequencies, `Lena' contains both low and high frequency contents, while `House' is rich in low frequencies.  

From Fig. \ref{TimeComparisonWithMeasurementMatrices}, we note that the reconstruction time with Gaussian and Bernoulli sensing matrices is almost the same, whereas reconstruction time with the PCI matrix is extremely less. This huge reduction in time is owing to the implementation simplicity of PCI matrix. However, there is a trade-off between the reconstruction time and the accuracy. Fig. \ref{PSNRComparisonWithMeasurementMatrices} compares reconstruction accuracy of images in terms of peak signal-to-noise ratio (PSNR) given by:
\begin{equation}
\text{PSNR (in dB)}=10\log_{10} \left( \frac{\sum\limits_{i=0}^{m-1} \sum\limits_{j=0}^{n-1} |\textbf{I}(i,j)-\hat{\textbf{I}}(i,j)|}{\sum\limits_{i=0}^{m-1} \sum\limits_{j=0}^{n-1} |\textbf{I}(i,j)|^2} \right ),
\label{Equation for PSNR}
\end{equation}
where $\textbf{I}$ is the reference image, $\hat{\textbf{I}}$ is the reconstructed image, and $m \times n$ is the size of the image. Results shown are averaged over 10 iterations. From Fig. \ref{PSNRComparisonWithMeasurementMatrices}, we observe that reconstruction accuracy with PCI matrix is approx. 2 dB lower than that with Bernoulli or Gaussian sensing matrices. Similar to this observation, in \cite{he2010simplest}, authors have shown comparatively inferior CS-based reconstruction results with PCI matrix. In section-\ref{Section for Experimental Results}, we will show that this gap in reconstruction accuracy can be not only bridged, but rather enhanced to a large extent with matched wavelet. Since the use of PCI matrix expedites the complete pipeline, its use is overall useful, if a little bit inferior quality can be bridged.
\section{Proposed Wavelet Decomposition Method on Images}  
\label{Section For Proposed Strategy}
\vspace{-0.9em}
In this Section, we propose a new optimized strategy of multi-level wavelet decomposition on images. 

A separable wavelet transform is implemented on images by first applying 1-D wavelet transform along the columns and then along the rows of an image. This provides 1-level wavelet decomposition that consists of four components labeled as LL, LH, HL and HH, respectively. The same procedure is repeated on the LL part of the wavelet transform $k$-times to obtain $k$-level decomposition of an image (Fig. \ref{Fig for existing wavelet transform}). We call this decomposition as regular pyramid (R-Pyramid) wavelet decomposition.

In general, $k$-level wavelet decomposition of an image consists of following components: LL\textsub{\textit{k}}, LH\textsub{\textit{i}}, HL\textsub{\textit{i}} and HH\textsub{\textit{i}}, where $i=1,2,...,k-1$. LH\textsub{i}, HL\textsub{i} and HH\textsub{i} components are obtained by applying wavelet transform on the columns and rows of LL\textsub{\textit{i}-1} component. LH\textsub{\textit{i}} is obtained by filtering LL\textsub{\textit{i}-1} column-wise using a lowpass filter and filtering it row-wise using a highpass filter. Thus, the current nomenclature of labeling subbands is: first character represents operation on columns and second character represents operation on rows, where operation implies highpass or lowpass filtering denoted by symbols `H' and `L', respectively. 
\begin{figure*}[!ht]
\vspace{-0.4em}
\centering
\begin{subfigure}[b]{0.24\textwidth}
\centering
\includegraphics[scale=0.45]{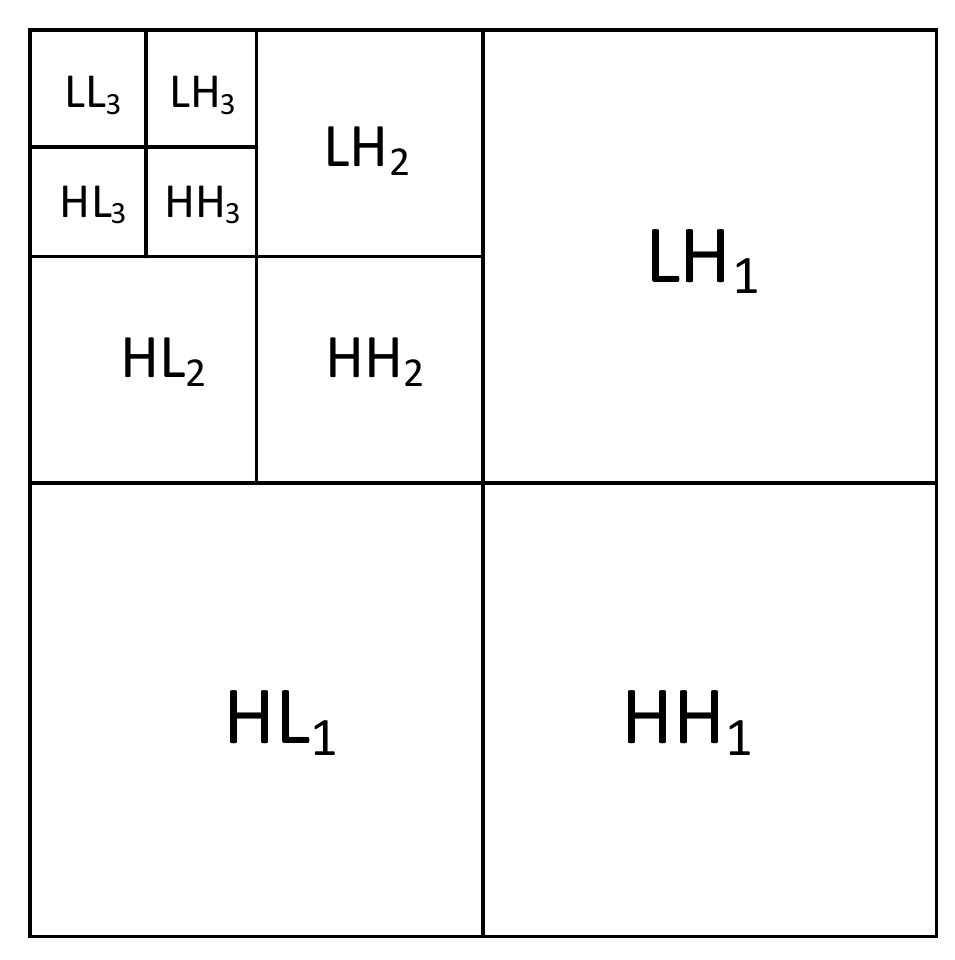}
\captionsetup{justification=centering}
\caption{3-level R-Pyramid wavelet decomposition}
\label{Fig for existing wavelet transform}
\end{subfigure}
\begin{subfigure}[b]{0.24\textwidth}
\centering
\includegraphics[scale=0.45]{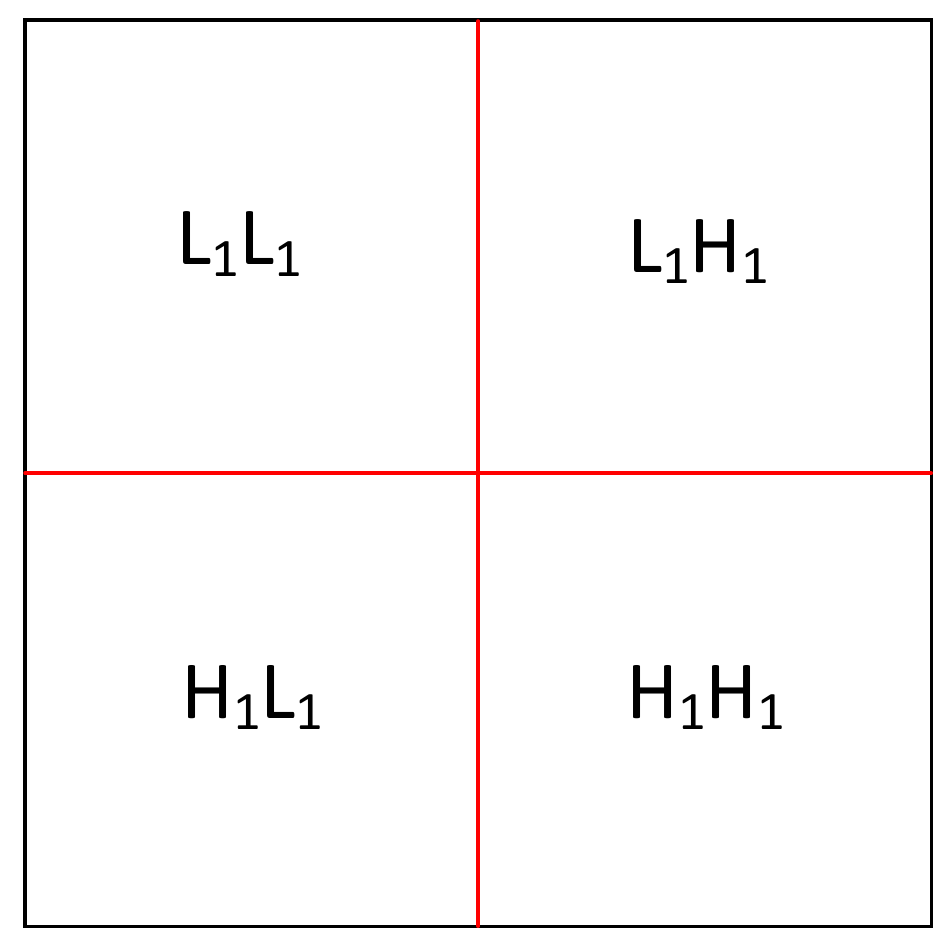}
\captionsetup{justification=centering}
\caption{1-level L-Pyramid wavelet decomposition}
\label{Fig for proposed level-1 wavelet transform}
\end{subfigure}
\begin{subfigure}[b]{0.24\textwidth}
\centering
\includegraphics[scale=0.45]{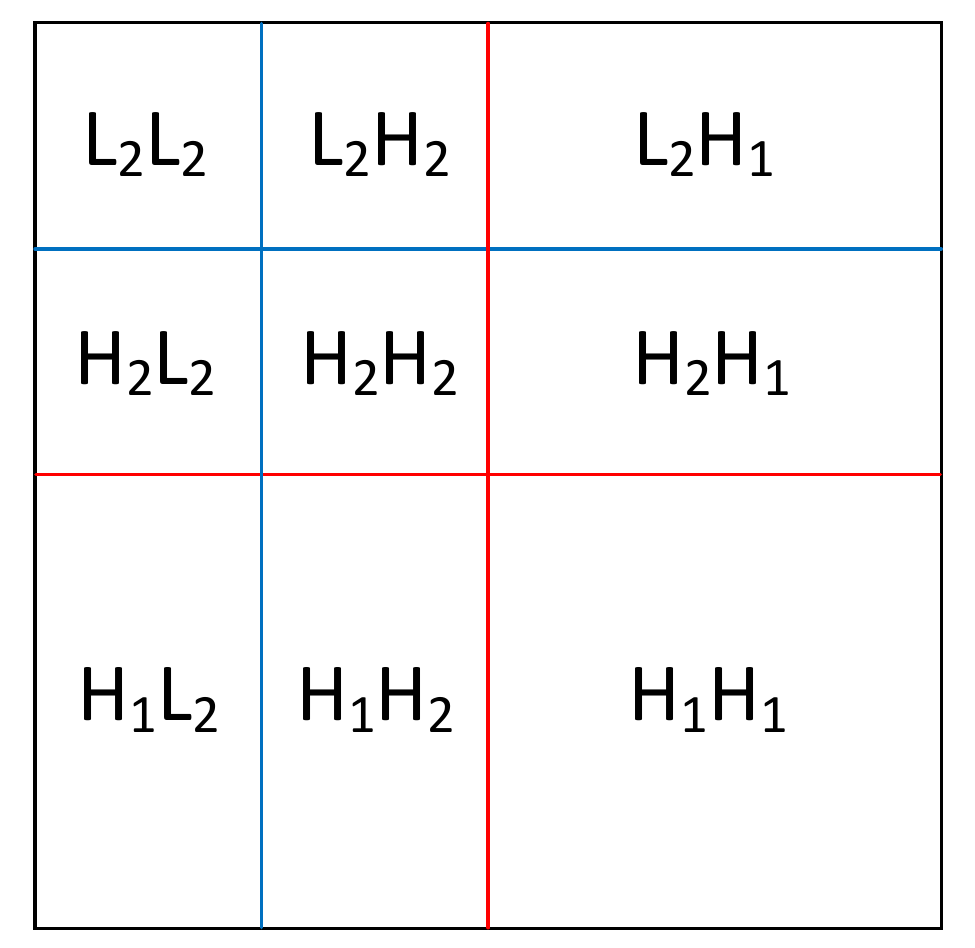}
\captionsetup{justification=centering}
\caption{2-level L-Pyramid wavelet decomposition}
\label{Fig for proposed level-2 wavelet transform}
\end{subfigure}
\begin{subfigure}[b]{0.24\textwidth}
\centering
\includegraphics[scale=0.45]{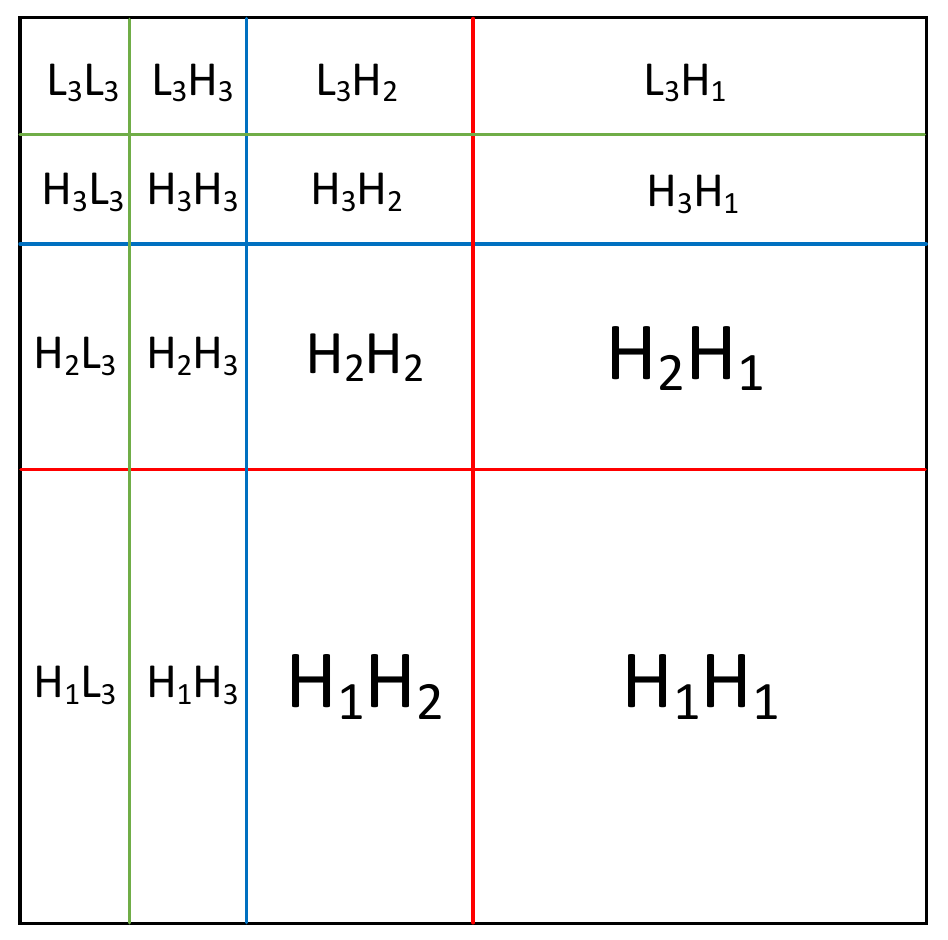}
\captionsetup{justification=centering}
\caption{3-level L-Pyramid wavelet decomposition}
\label{Fig for proposed level-3 wavelet transform}
\end{subfigure}
\vspace{-0.5em}
\captionsetup{justification=centering}
\caption{\small Multi-level wavelet decomposition of image}
\label{Fig forWavelet transform}
\vspace{-1.2em}
\end{figure*}
\begin{figure*}[!ht]
\centering
\begin{subfigure}[b]{0.32\textwidth}
\centering
\includegraphics[scale=0.52]{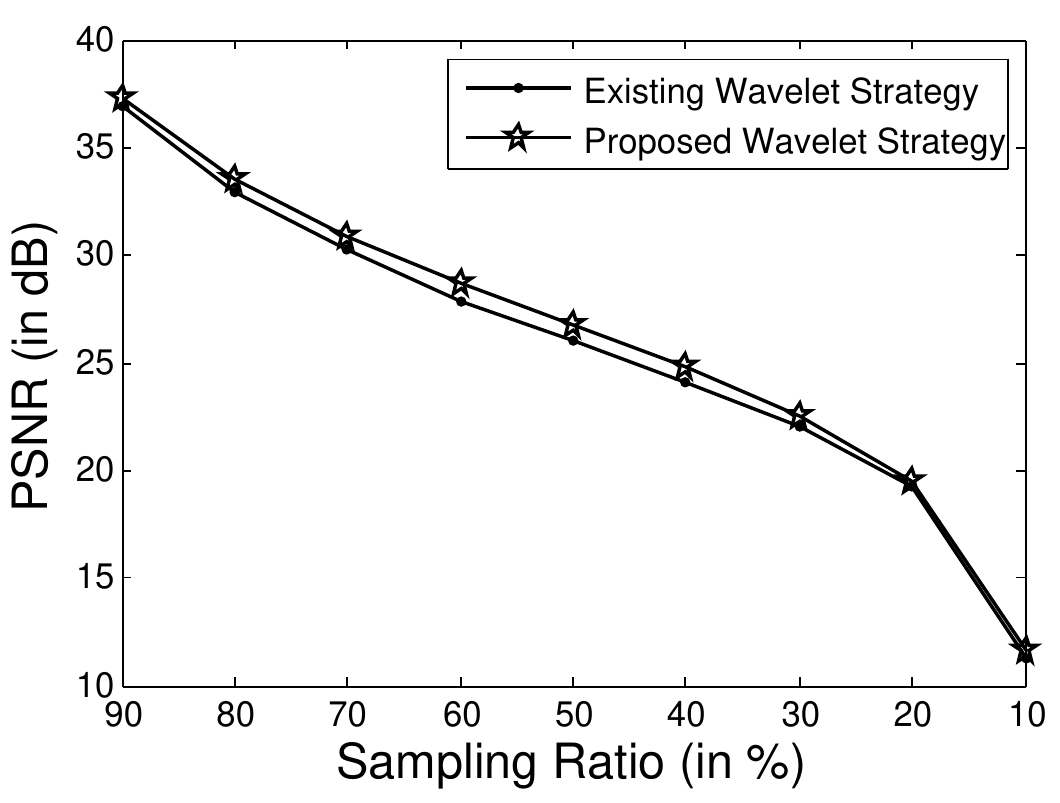}
\caption{}
\end{subfigure}
\begin{subfigure}[b]{0.32\textwidth}
\centering
\includegraphics[scale=0.52]{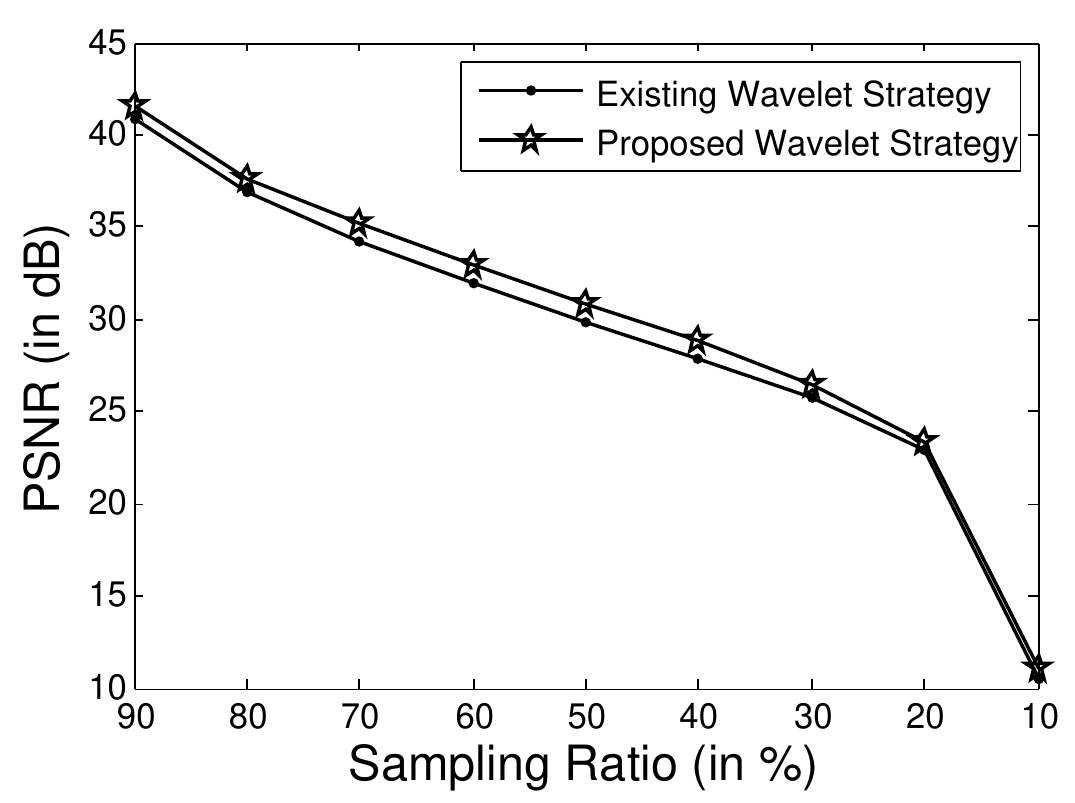}
\caption{}
\end{subfigure}
\begin{subfigure}[b]{0.32\textwidth}
\centering
\includegraphics[scale=0.52]{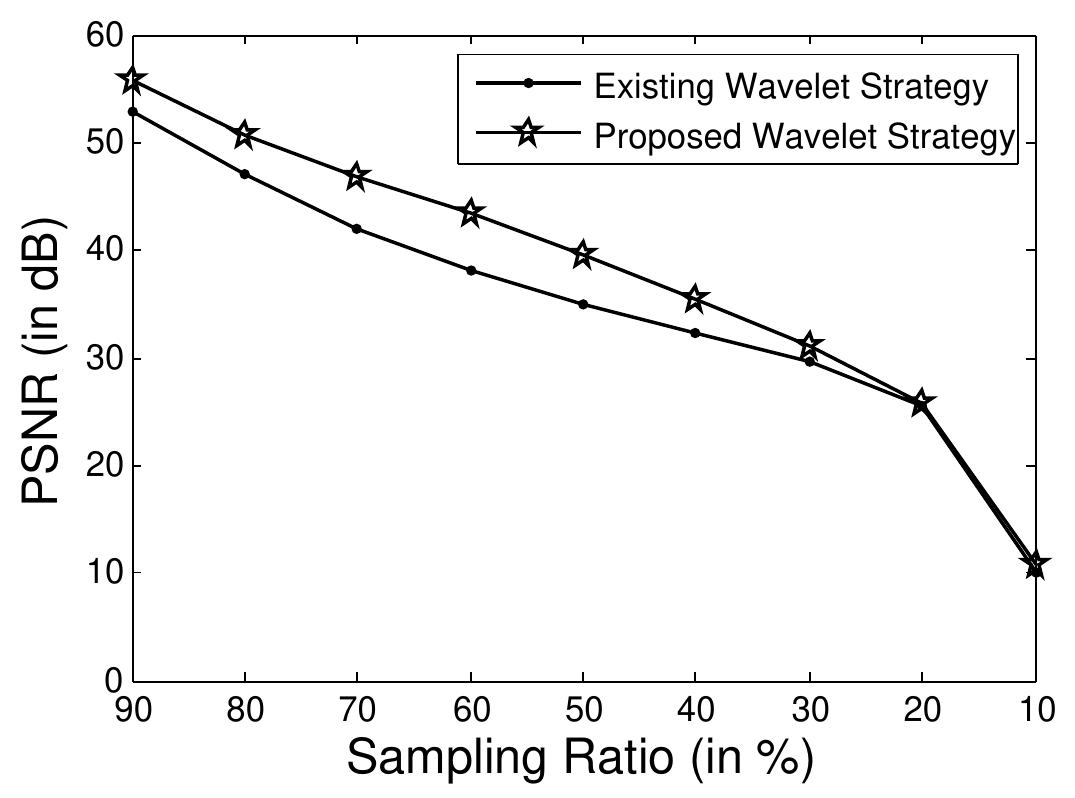}
\caption{}
\end{subfigure}
\vspace{-0.8em}
\captionsetup{justification=centering}
\caption{\small CS-based reconstruction accuracy with the existing and the proposed wavelet decomposition strategy with `db4' wavelet on image (a) `Beads' (b) `Lena'  (c) `House'} 
\vspace{-1.8em}
\label{Fig for strategy comparison}
\end{figure*}
In the conventional 2-D wavelet transform (Fig. \ref{Fig for existing wavelet transform}), wavelet decomposition is applied on LL\textsub{\textit{i}} part only to obtain the $(i+1)^{th}$ level coefficients. Since it is a separable transform, similar to 1-D wavelet transform wherein wavelet is applied repeatedly on lowpass filtered branches, we propose to apply wavelet in the lowpass filtered directions of LH\textsub{\textit{i}-1} and HL\textsub{\textit{i}-1} subbands in contrast to the conventional decomposition strategy wherein these subbands are left unaltered. Thus, the proposed 2nd level wavelet decomposition is as shown in Fig. \ref{Fig for proposed level-2 wavelet transform}.

Since we apply wavelet in only one direction of LH\textsub{\textit{i}-1} and HL\textsub{\textit{i}-1} subbands, we notate these subbands differently compared to the conventional scheme. We assign subscript with both `L' and `H' symbols of every subband to denote the no. of times wavelet has been applied in that direction. In order to understand this, let us first consider 1-level wavelet decomposition as shown in Fig. \ref{Fig for proposed level-1 wavelet transform} that is similar to the conventional scheme shown in Fig. \ref{Fig for existing wavelet transform}. However, the subbands are labeled as L\textsub{1}L\textsub{1}, L\textsub{1}H\textsub{1}, H\textsub{1}L\textsub{1}, and H\textsub{1}H\textsub{1}.  

In the 2nd level wavelet decomposition, wavelet is applied in both directions of L\textsub{1}L\textsub{1} subbands leading to L\textsub{2}L\textsub{2}, L\textsub{2}H\textsub{2}, H\textsub{2}L\textsub{2}, and H\textsub{2}H\textsub{2} subbands. But in addition, wavelet is applied on the columns of L\textsub{1}H\textsub{1} yielding two subbands L\textsub{2}H\textsub{1} and H\textsub{2}H\textsub{1}. Also, wavelet is applied on the rows of H\textsub{1}L\textsub{1} subband yielding two subbands H\textsub{1}L\textsub{2} and H\textsub{1}H\textsub{2}. Applying similar strategy for the 3rd level decomposition, we obtain subbands as shown in Fig. \ref{Fig for proposed level-3 wavelet transform}. We name this wavelet decomposition as L-shaped pyramid (L-Pyramid) wavelet decomposition.

The efficacy of the proposed wavelet strategy is shown in CS-based image reconstruction using orthogonal Daubechies wavelet `db4'. Fig. \ref{Fig for strategy comparison} shows reconstruction accuracy in terms of PSNR \eqref{Equation for PSNR} with sampling ratios ranging from 10\% to 90\% averaged over 10 iterations. We compare reconstruction accuracy with the existing wavelet decomposition strategy and with the proposed strategy on the same three images: `Beads', `Lena' and `House'. From Fig. \ref{Fig for strategy comparison}, we note marked improvement with the proposed strategy over the existing strategy. Performance is particularly improved for image `House' that is rich in low frequencies.

\section{Proposed Method of Designing Matched Wavelet from Compressively Sensed Images}
\label{Section For Proposed Method of Matched Wavelet}
\vspace{-0.6em}
As discussed earlier, wavelets are extensively used as sparsifying transform for image reconstruction in CS domain. Researchers often face challenge in choosing a wavelet because it is not known apriori as to which wavelet will provide better representation of the signal. In general, one uses compactly supported wavelets, either orthogonal Daubechies wavelets or biorthogonal 9/7 or 5/3 wavelets (note that first digit denotes the length of analysis lowpass filter and second digit denotes the length of analysis highpass filter). It is natural to think that a wavelet matched to a given signal will provide best representation of the given signal and hence, will provide better reconstruction in CS compared to any wavelet chosen arbitrarily. Although a number of methods exist for designing wavelet matched to a given signal \cite{gupta2005new1,gupta2005new2,ansari2015signal,chapa2000algorithms, claypoole1998adaptive,piella2005gradient,heijmans2002building}, these methods work when full signal is available. On the other hand, in CS, only compressively sensed signal is available at the receiver. No method exists for estimation of matched wavelet from compressively sensed signal that can be utilized at the same time for efficient reconstruction of signal/image from compressively sensed signal/image. In this Section, we address this problem. We propose a joint framework for signal reconstruction in CS wherein we are estimating wavelet from the compressivley sensed image and use it for efficient image reconstruction at the same time.

Since the proposed work is on separable wavelets, we require to estimate matched wavelet for row and column directions separately. Thus, before proceeding with the matched wavelet design, we present the scanning mechanism of rows and columns data in images as used in this work.
\subsection{Scanning Mechanism for Row- and Column-wise Data}
\label{For application Of Matched Wavelet to Images}
As stated earlier, we require to estimate matched wavelet for both the row and column directions. One easier strategy can be designing matched wavelet on row- or column-vectorized image and use the same wavelet, later, along both the columns and rows as a separable wavelet. Instead, we propose to design matched wavelet separately for the row- and column-directions. These directions can be scanned using the following two strategies:
\begin{figure}[!ht]
\centering
\begin{subfigure}[b]{0.23\textwidth}
\vspace{-0.6em}
\centering
\includegraphics[scale=0.4]{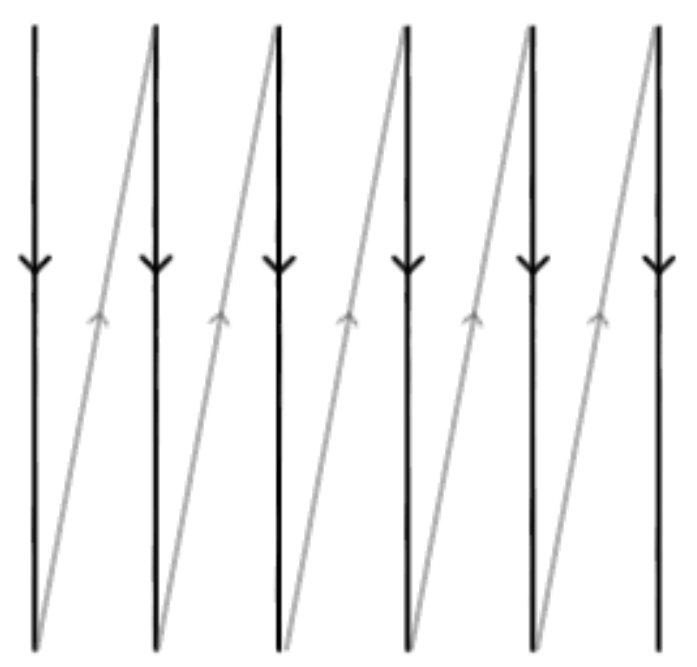}
\caption{Column wise scanning}
\label{Fig for Raster scan way1 column wise}
\end{subfigure}
\begin{subfigure}[b]{0.23\textwidth}
\centering
\includegraphics[scale=0.4]{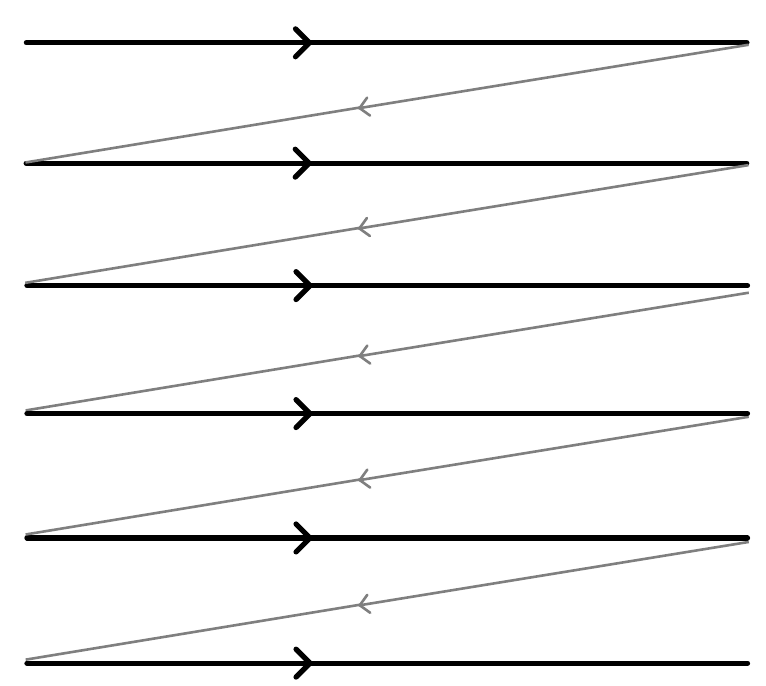}
\caption{Row wise scanning}
\label{Fig for Raster scan way1 row wise}
\end{subfigure}
\captionsetup{justification=centering}
\vspace{-0.4em}
\caption{\small Scanning Strategy-1}
\label{Fig for Raster scan way1}
\vspace{-1.4em}
\end{figure}
\begin{figure}[!ht]
\centering
\begin{subfigure}[b]{0.23\textwidth}
\centering
\includegraphics[scale=0.46]{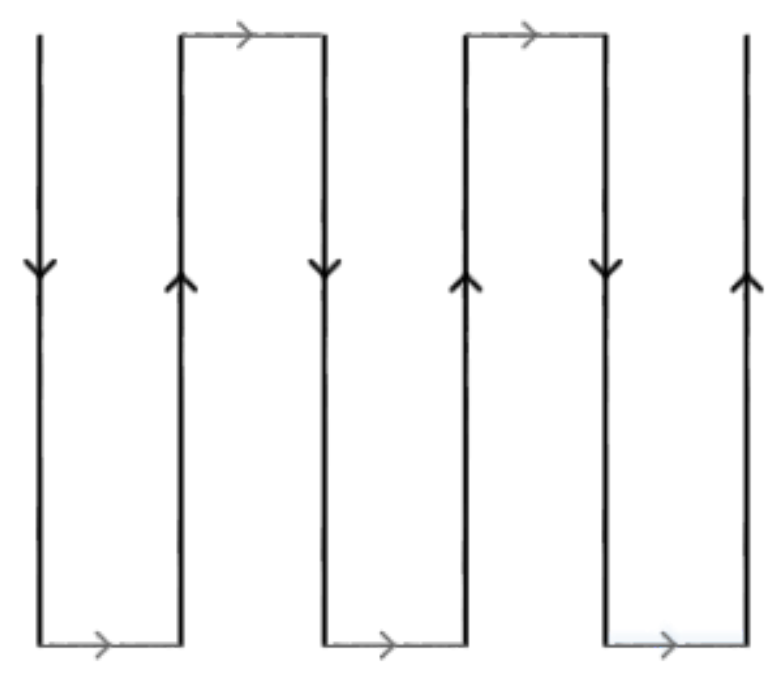}
\caption{Column wise scanning}
\label{Fig for Raster scan way2 column wise}
\end{subfigure}
\begin{subfigure}[b]{0.23\textwidth}
\centering
\includegraphics[scale=0.46]{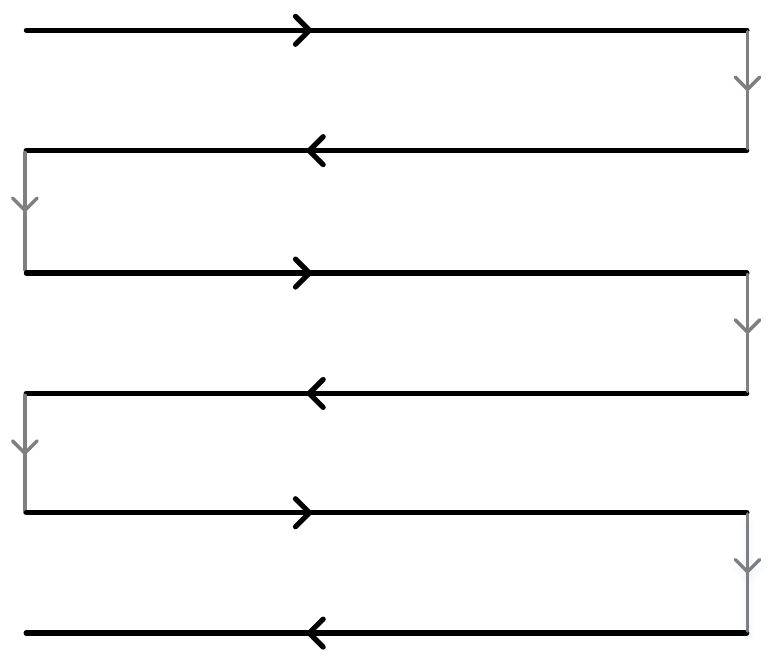}
\caption{Row wise scanning}
\label{Fig for Raster scan way2 row wise}
\end{subfigure}
\captionsetup{justification=centering}
\vspace{-0.4em}
\caption{\small Scanning Strategy-2}
\label{Fig for Raster scan way2}
\vspace{-1em}
\end{figure}
\begin{itemize}
\item{Strategy-1}: 
We scan the image according to the scanning scheme as shown in Fig. \ref{Fig for Raster scan way1}, wherein rows or columns are stacked one after the other to obtain 1-D signal for both the directions. However, this will cause discontinuity in the 1-D signal at the transitions when one column ends and another starts and likewise, for the rows.  

\item{Strategy-2}: In order to avoid this discontinuity, other alternate strategy is to scan all even rows or columns in the reverse direction as shown in Fig. \ref{Fig for Raster scan way2}.
\end{itemize}

Since strategy-2 is robust to sudden transitions at the row- or column-endings, we use it in all our experiments. 
\vspace{-1.7em}
\subsection{Proposed Methodology of Matched Wavelet Design}
\vspace{-0.6em}
With the scanning strategy-2 proposed above, we convert a given image into two 1-D signals: one with column-wise scanning and another with row-wise scanning. Henceforth, in the present Section, we present matched wavelet design methodology using compressively sensed 1-D signal. We will use these designs in the next Section to note the performance over images as separable wavelet transforms. 

The proposed methodology has three stage. In stage-1, we obtain coarse image estimate from compressively (partially) sensed samples using a standard wavelet. We call this a coarser estimate because the wavelet used is not matched to the given signal and hence, the original signal may not be that sparse over this wavelet compared to that with the matched wavelet. This will impact the reconstruction performance. In stage-2, we estimate matched analysis wavelet filter that provides sparser subband wavelet (detail) coefficients than those obtained from the standard wavelet in stage-1. Using these estimated filters and the coarser signal estimate of stage-1, we design all filters of the matched wavelet system. In stage-3, we reconstruct signal from measured sub-samples using the matched wavelet estimated in stage-2.

\subsubsection{Stage-1: Coarser Image Estimation}
\label{Section For Step-I}
In this stage, we reconstruct a coarser estimate of the signal from compressively (partially) sensed measurement data $\textbf{y}$ using Basis pursuit (BP) optimization method as below:
\begin{equation}
\tilde{\mathbf{s}} = \min_{\mathbf{s}} ||\mathbf{s}||_1 \text{     subject to:     } \mathbf{y}=\mathbf{\Phi \Psi} \mathbf{s}.
\label{stage-1}
\end{equation}
where $\mathbf{\Psi}$ corresponds to any standard wavelet. We use biorthogonal 5/3 wavelet. The coarser approximation of the signal is obtained as $\tilde{\mathbf{x}}=\mathbf{\Psi} \mathbf{\tilde{s}}$. We solve the above optimization problem using MATLAB solver spgl1 \cite{BergFriedlander:2008,spgl1:2007}. 
\subsubsection{Stage-2: Estimation of matched wavelet}
\label{Section for Step-II}
We use the signal reconstructed in the previous step to design matched wavelet system by designing Predict and update stage filters of the lifting structure. 

\textbf{(a) Predict Stage: }
For designing signal-matched analysis wavelet filter in the lifting framework, first, we consider Lazy wavelet with $H_0(z)=1$, $H_1(z)=z$, $F_0(z)=1$, and $F_1(z)=z^{-1}$, in Fig. \ref{WaveletFilterbank}. We start with this Lazy wavelet structure and proceed with designing the predict stage filter $T(z)$. 

We apply coarser version of original signal $\tilde{\textbf{x}}$ estimated in Stage-1 as input to the Lazy wavelet filterbank and obtain even and odd sampled streams $\tilde{\textbf{x}}_e(n)$ and $\tilde{\textbf{x}}_o(n)$, respectively as shown in Fig. 1(a). We pass even indexed samples $\tilde{\textbf{x}}_e(n)$ through the predict stage filter and write the output of the lower subband signal, shown in Fig. 1(a), as below:
\begin{align}
\tilde{d}_{-1}[n]&=\tilde{x}_o[n]-\tilde{x}_e[n] \ast t[n] \nonumber \\
&=\tilde{x}_o[n]-p[n] \nonumber \\
&=\tilde{x}[2n+1]-p[n]
\label{eq:no16}
\end{align}
where
\begin{equation}
p[n]=\sum_{k=0}^{L_t-1}t[k]\tilde{x}_e[n-k],
\label{Equation for Prediction output}
\end{equation}    
where `$\ast$' is the convolution operator, $L_t$ is the length of the predict stage filter $t[n]$ with its \textit{Z}-transform given by $T(z)=Z\{t[n]\}$. For good prediction, a sample (here odd indexed) should be predicted from its immediate past and immediate future neighbors that requires a careful choice on the predict stage filter provided by Theorem-1 as below.
\vspace{-0.5em}
\begin{theorem}
The following structure of predict stage filter allows odd-indexed samples to be predicted from their nearest even-indexed samples, i.e., from their immediate past and immediate future samples:
\begin{equation}
T(z)=z^{-(\frac{L_t}{2}-1)} \sum_{i=0}^{L_t-1} t[i]z^i,
\label{PredictStageFilterStructure}
\end{equation}
where $T(z)$ is an even-length filter. Even length $T(z)$ will ensure that equal number of past and future samples are used in prediction.
\end{theorem}
\vspace{-0.5em}
\begin{proof}
On using \eqref{PredictStageFilterStructure} in \eqref{Equation for Prediction output}, we obtain:
\begin{equation}
\label{eq19}
p[n]=\sum_{k=0}^{L-1} t[k]\tilde{x}[2n-L_t+2+2k].
\end{equation}
\begin{figure*}[!ht]
\vspace{-1.4em}
\centering
\includegraphics[scale=0.9]{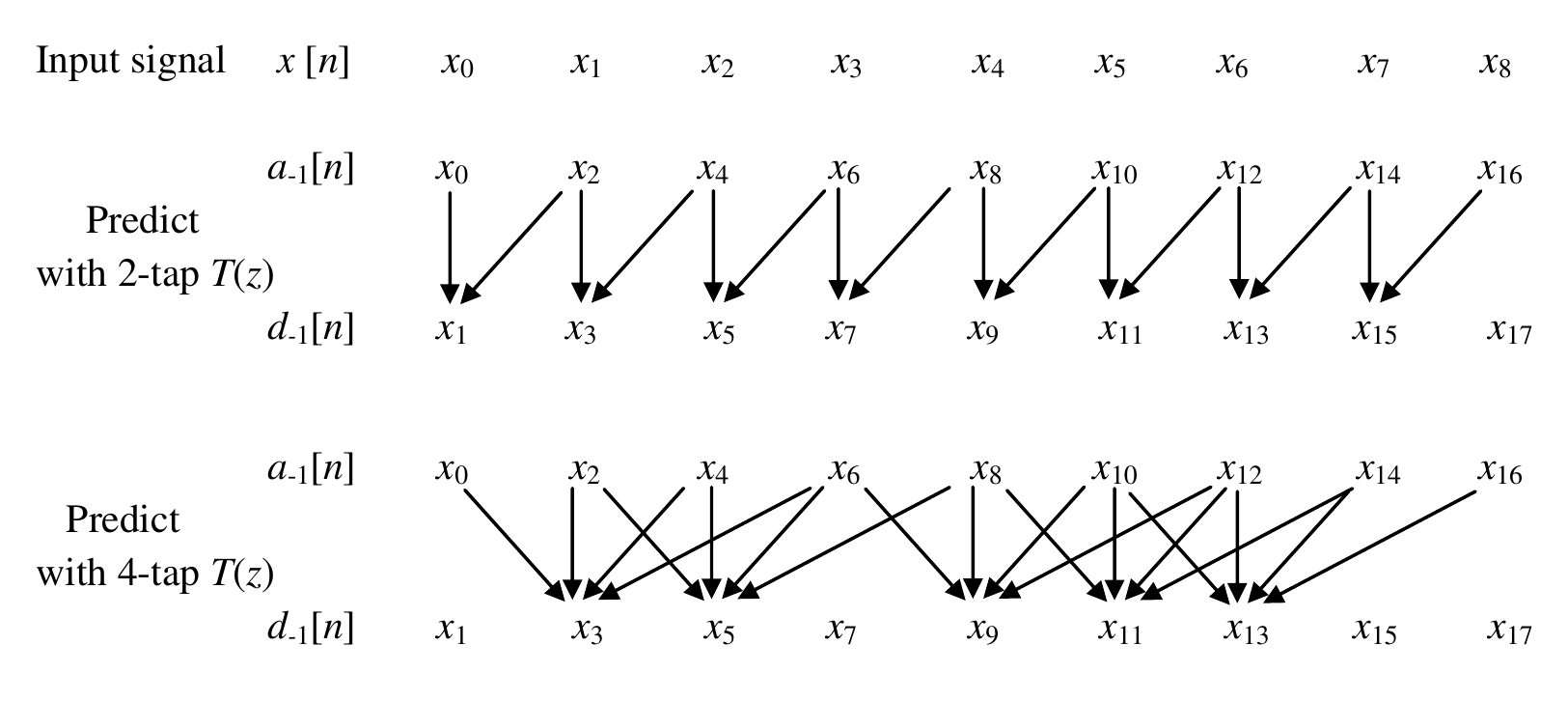}
\vspace{-1.4em}
\caption{An odd sample is being predicted from its neighboring even samples with a 2-tap and a 4-tap filter $T(z)$}
\label{Fig for predict diagram}
\vspace{-1em}
\end{figure*}
\begin{figure*}[!ht]
\centering
\includegraphics[scale=0.9]{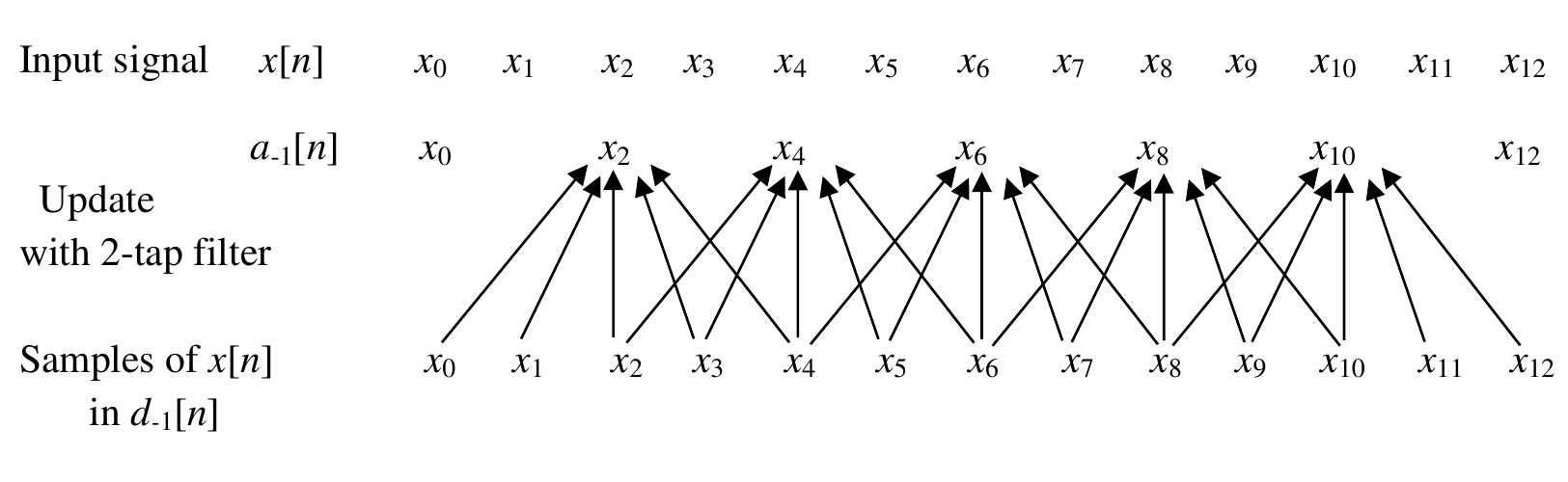}
\vspace{-1.6em}
\caption{Even samples being updated from neighboring samples with a 2-tap $T(z)$ and 2-tap update filter $S(z)$}
\vspace{-1.4em}
\label{Fig for update diagram}
\end{figure*}

This is to note that the $n^{th}$ sample of $p[n]$ predicts sample $\tilde{x}[2n+1]$. On expanding \eqref{eq19}, we obtain
\begin{align}
p[n]=&t[0]\tilde{x}[2n-L_t+2]+...+t[{\frac{L_t}{2}-1}]\tilde{x}[2n] \nonumber \\
&+t[{\frac{L_t}{2}}]\tilde{x}[2n+2]+...+t[{L_t-1}]\tilde{x}[2n+L_t].
\end{align}
Thus, the sample $\tilde{x}[2n+1]$ is predicted from its exact nearest even neighbors. For better clarity, Fig. \ref{Fig for predict diagram} shows even neighboring samples that are being used to predict odd samples with a 2-tap and 4-tap predict filter $T(z)$. 
This proves the theorem.
\end{proof}

Subband signal $\tilde{\mathbf{d}}_{-1}$ can be considered as the noisy version of detail coefficients $\mathbf{d}_{-1}$ that could be obtained by passing the original signal \textbf{x} through the signal-matched analysis wavelet branch. This can be written as:
\begin{equation}
\tilde{\mathbf{d}}_{-1}=\mathbf{d}_{-1}+\bm{\eta}_1,
 \label{noisy wavelet coefficients-1}
\end{equation}
where $\bm{\eta}_1$ is the corresponding error.

Further, we analyze this coarser approximate signal $\tilde{\textbf{x}}$ using the same biorthogonal 5/3 wavelet with which it had been reconstructed and obtain detail coefficients (subband coefficients of highpass filtered branch) $\hat{\mathbf{d}}_{-1}$. Again, this signal can be considered as the noisy version of detail coefficients $\mathbf{d}_{-1}$ and can be written as below:
\begin{equation}
\hat{\mathbf{d}}_{-1}=\mathbf{d}_{-1}+\bm{\eta}_2,
 \label{noisy wavelet coefficients-2}
\end{equation}
where $\bm{\eta}_2$ is the corresponding error. 

From \eqref{noisy wavelet coefficients-1} and \eqref{noisy wavelet coefficients-2}, we obtain:
\begin{equation}
\hat{\mathbf{d}}_{-1}=\tilde{\mathbf{d}}_{-1}+\bm{\eta},
 \label{noisy wavelet coefficients-3}
\end{equation}
where $\bm{\eta}$ is the error that consists of two components: 1) because we are using approximate signal $\tilde{\mathbf{x}}$ instead of the original signal \textbf{x} and 2) because we are using biorthogonal 5/3 wavelet instead of the signal-matched wavelet. 

On substituting for $\tilde{\mathbf{d}}_{-1}$ from \eqref{eq:no16}, \eqref{Equation for Prediction output}, and \eqref{PredictStageFilterStructure} in \eqref{noisy wavelet coefficients-3} and writing in the matrix form, we obtain
\begin{equation}
{\hat{\mathbf{d}}}_{-1}=\mathbf{At}+\bm{\eta},
\label{eq22}
\end{equation}
where $\mathbf{A}$ is the convolution matrix consisting of even and odd indexed samples of $\tilde{\textbf{x}}$ and $\textbf{t}$ denotes the vectorized form of predict stage filter $t[n]$ or $T(z)$. We solve for $\textbf{t}$ in \eqref{eq22} using least squares method and substitute in \eqref{eq:no1} and \eqref{eq:no2} to update the analysis highpass and synthesis lowpass filters and obtain new filters $H_1^{new}(z)$ and $F_0^{new}(z)$, respectively. This ends the predict stage.

\textbf{(b) Update Stage: }
In the update stage, update polynomial $S(z)$ is required to be computed. In order to do this, we write the output of the upper subband signal using the lower subband signal, $\tilde{d}_{-1}[n]$ (refer Fig. 1) as below:
\begin{equation}
a_{-1}[n]=x_e[n]+\tilde{d}_{-1}[n] \ast s[n],
\label{eq_update}
\end{equation} 
where $s[n]$ is the time domain description of the update stage filter $S(z)$. 

Similar to the predict stage filter, we require to choose $s[n]$ such that the elements of the upper branch are updated using nearest neighbors only. The corresponding structure for $s[n]$ is provided by Theorem-2 as below.

\begin{theorem}
The following structure of the update stage filter allows the elements of the upper branch to be updated from nearest neighbors: 
\begin{equation}
S(z)=z^{(\frac{L_s}{2}-1)}\sum_{i=0}^{L_s-1} s[i]z^{-i},
\label{UpdateStageFilterStructure}
\end{equation}
where $L_s$ is the length of filter $S(z)$ or $s[n]$. Note that $S(z)$ is an even-length filter that ensures that sample update is done using equal number of past and future samples.
\end{theorem} 
\begin{proof}
For the sake of simplicity, let us consider two tap predict stage filter that provides the following detail coefficients:
\begin{equation}
\tilde{d}_{-1}[n]=-t[0]\tilde{x}[2n]+\tilde{x}[2n+1]-t[1]\tilde{x}[2n+2]
\end{equation}
Once this signal is passed through the update stage filter in the update branch, we obtain:
\begin{equation}
u[n]=\tilde{d}_{-1}[n] \ast s[n],
\end{equation}
where `$\ast$' is the convolution operator. 
With the choice of filter in \eqref{UpdateStageFilterStructure}, we obtain:
\begin{align}
\label{update_2}
u[n]=&s[0]\tilde{d}_{-1}[n+\frac{L_s}{2}-1]+s[1]\tilde{d}_{-1}[n+\frac{L_s}{2}-2]+... \nonumber \\
&+s[L_s-2]\tilde{d}_{-1}[n-\frac{L_s}{2}+1]+s[L_s-1]\tilde{d}_{-1}[n-\frac{L_s}{2}].
\end{align}
\begin{figure*}[!ht]
\centering
\vspace{-1em}
\begin{subfigure}[b]{0.19\textwidth}
\includegraphics[scale=0.19]{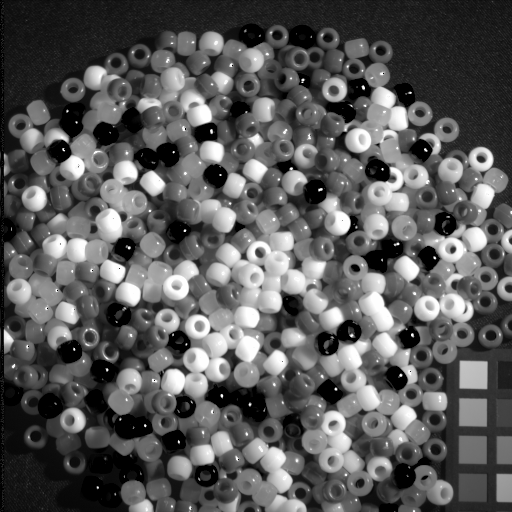}
\caption{Beads}
\label{Beads}
\end{subfigure}
\begin{subfigure}[b]{0.19\textwidth}
\includegraphics[scale=0.19]{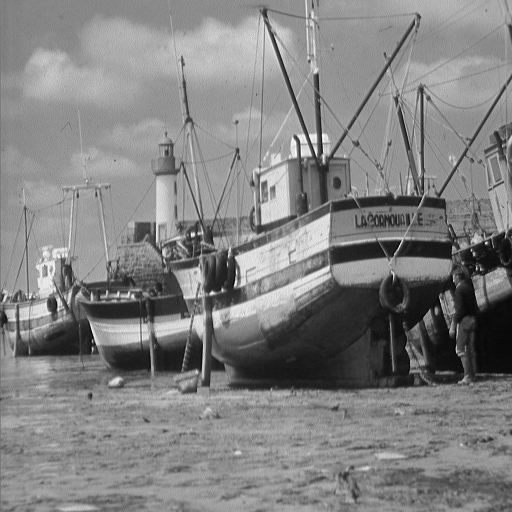}
\caption{Boat}
\label{Boat}
\end{subfigure}
\begin{subfigure}[b]{0.19\textwidth}
\includegraphics[scale=0.19]{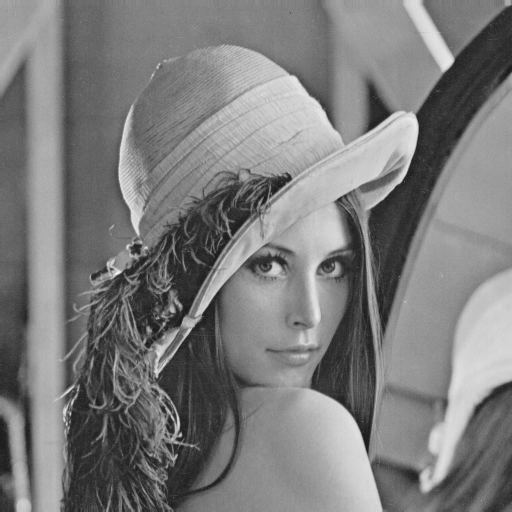}
\caption{Lena}
\label{Lena}
\end{subfigure}
\begin{subfigure}[b]{0.19\textwidth}
\includegraphics[scale=0.19]{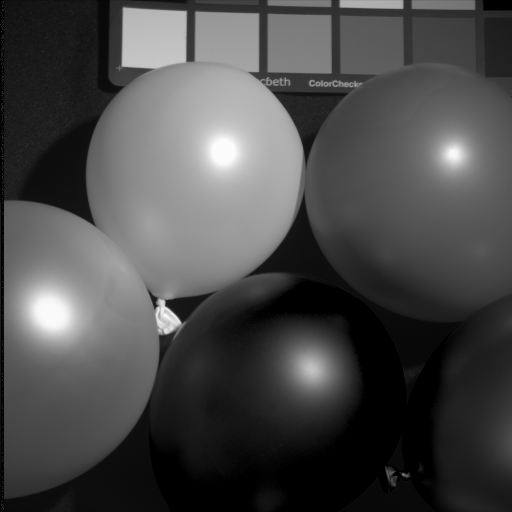}
\caption{Balloon}
\label{Balloon}
\end{subfigure}
\begin{subfigure}[b]{0.19\textwidth}
\includegraphics[scale=0.19]{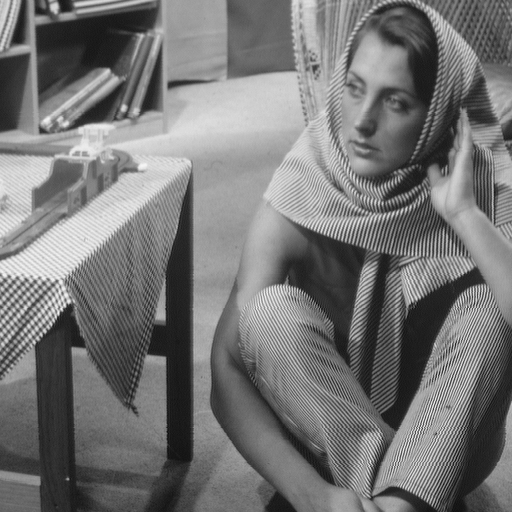}
\caption{Barbara}
\label{Barbara}
\end{subfigure}
\newline
\begin{subfigure}[b]{0.19\textwidth}
\includegraphics[scale=0.19]{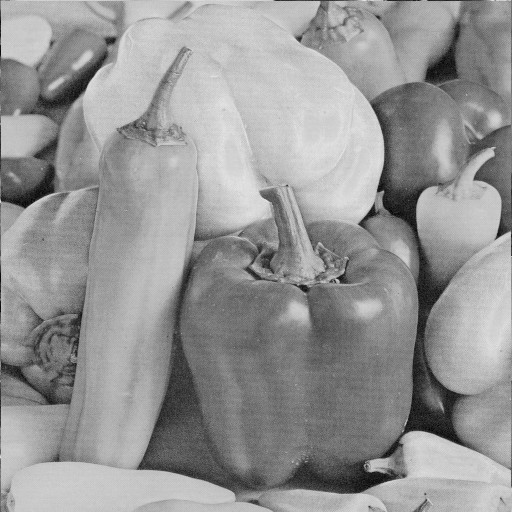}
\caption{Peppers}
\label{Peppers}
\end{subfigure}
\begin{subfigure}[b]{0.19\textwidth}
\includegraphics[scale=0.19]{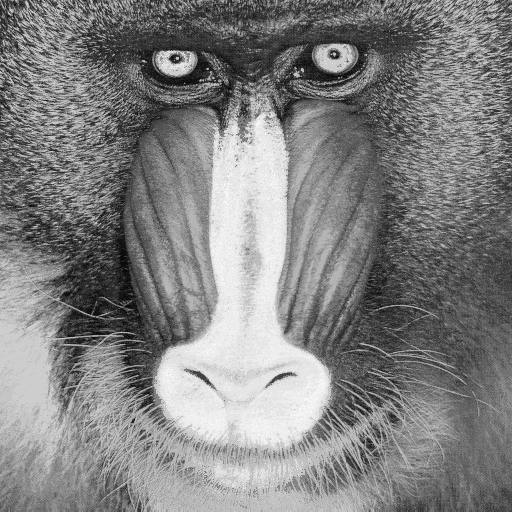}
\caption{Mandril}
\label{Mandril}
\end{subfigure}
\begin{subfigure}[b]{0.19\textwidth}
\includegraphics[scale=0.25]{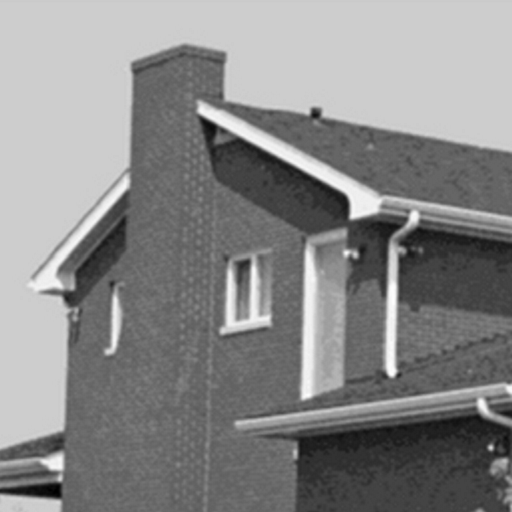}
\caption{House}
\label{House}
\end{subfigure}
\begin{subfigure}[b]{0.19\textwidth}
\includegraphics[scale=0.25]{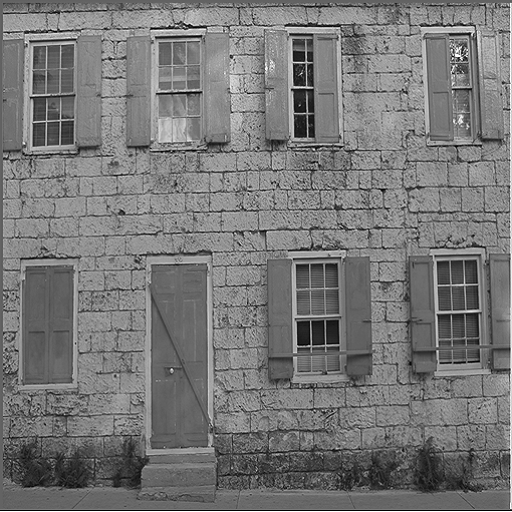}
\caption{Building}
\label{Building}
\end{subfigure}
\begin{subfigure}[b]{0.19\textwidth}
\includegraphics[scale=0.25]{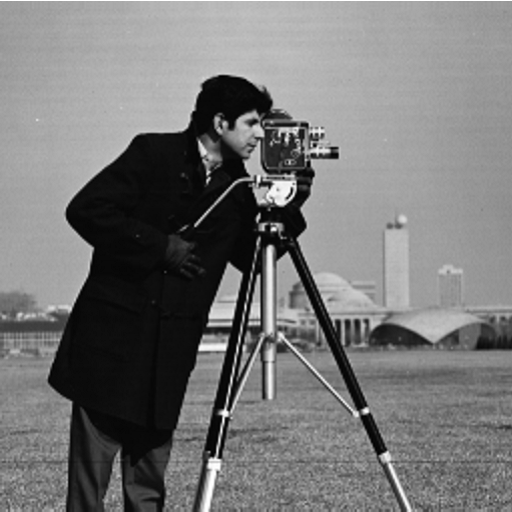}
\caption{Cameraman}
\label{Cameraman}
\end{subfigure}
\vspace{-0.6em}
\caption{Images used in experiments}
\label{Fig for All Images}
\vspace{-1.6em}
\end{figure*}
On expanding \eqref{update_2} and rearranging, we obtain
\begin{align}
u[n]&=-s[0]t[1]\tilde{x}[2n+L_s]+s[0]\tilde{x}[2n+L_s-1]+\nonumber \\
&(-s[0]t[0]-s[1]t[1])\tilde{x}[2n+L_s-2]+...+\nonumber \\
&(-s[L_s-2]t[0]-s[L_s-1]t[1])\tilde{x}[2n-L_s+2]+\nonumber \\
&s[L_s-1]\tilde{x}[2n-L_s+1]-s[L_s-1]t[0]\tilde{x}[2n-L_s]
\end{align}
The above signal updates approximate coefficients, i.e., $\tilde{x}[2n]$. It can be clearly noticed that coefficients $\tilde{x}[2n]$ are updated using the nearest neighbors from $\tilde{x}[2n+L_s]$ to $\tilde{x}[2n-L_s]$. For better clarity, Fig. \ref{Fig for update diagram} shows the neighboring samples that are being used to update even samples with a 2-tap predict filter $T(z)$ and a 2-tap update filter $S(z)$. 
\end{proof}

This subband signal $a_{-1}[n]$ is passed through a 2-fold upsampler that provides:
\begin{equation}
 \tilde{x}_{1u}[n] =
  \begin{cases}
    a_{-1}[\frac{n}{2}]        & \quad \text{if } n \text{ is a multiple of 2}\\
    0  & \quad \text{otherwise.}\\
  \end{cases}
  \label{eq14}
\end{equation}

Next, signal $\tilde{x}_{1u}[n]$ is passed through the synthesis lowpass filter $f^{new}_0[n]$ that was updated in the predict stage mentioned earlier. This provides us the signal $\tilde{x}_{1}[n]$ (shown in Fig. 2) reconstructed from the upper subband only and is given by
\begin{equation}
\tilde{x}_{1}[n]=\tilde{x}_{1u}[n] \ast f^{new}_0[n].
\label{eq15}
\end{equation}    

Assuming that the original signal of interest is rich in low frequency content, signal $\tilde{x}_{1}[n]$ reconstructed in the upper subband should be in close approximation to the input signal $\tilde{\textbf{x}}$. This allows us to solve for the update stage filter as below:

\begin{equation}
\mathbf{\hat{s}}=\min_{\mathbf{s}} \sum_{n} (\tilde{x}_{1}[n]-\tilde{x}[n])^2.
\label{eq16}
\end{equation} 

It can be noted from \eqref{eq_update}, (\ref{eq14}) and (\ref{eq15}), that $\tilde{\textbf{x}}_{1}$ can be written in terms of update stage filter $s[n]$ obtained on solving \eqref{eq16} using least squares method. Correspondingly, analysis lowpass filter $H_0(z)$ and synthesis highpass filter $F_1(z)$ are updated to $H^{new}_0(z)$ and $F^{new}_1(z)$ using \eqref{eq:no3} and \eqref{eq:no4}, respectively. This completes our design of matched wavelet. 

This is to be noted that, since the lifting scheme is modular, more number of such predict and update stage filters can be estimated and appended in order to design higher order or larger length filters.
  
\subsubsection{Stage-3: Signal Reconstruction using Matched Wavelet}
\label{Section For Step-3}
Once we have estimated matched wavelet, we employ \eqref{CS solution} on measured subsampled measured signal $\textbf{y}$ using CS with matched wavelet in $\textbf{W}$ to recover the original signal $\textbf{x}$. 

The above matched wavelet design for 1-D signals is applied on row-scanned and column-scanned signal of a given image and corresponding matched wavelet is designed for the row- and the column-space separately and hence, completes the design of image-matched separable wavelet system. Also, this is to note that the proposed method of designing matched wavelet from compressively sensed images is independent of the sensing matrix. The proposed matched wavelet can be designed from images compressively sensed using any sensing matrix.
\vspace{-0.8em}
\section{Experiments and Results}
\label{Section for Experimental Results}
In this section, we present CS-based reconstruction results of images using image-matched wavelets designed from compressively sensed images. We apply the proposed method on ten natural images shown in Fig. \ref{Fig for All Images}. Images with different spectral contents have been selected. For example, Balloon, House and cameraman are rich in low frequencies, Beads, Barbara and Mandrill are rich in high frequencies, and rest of the images have varied lower and higher frequency content.

\begin{figure*}[!ht]
\centering
\begin{subfigure}{0.48\textwidth}
\includegraphics[scale=0.58]{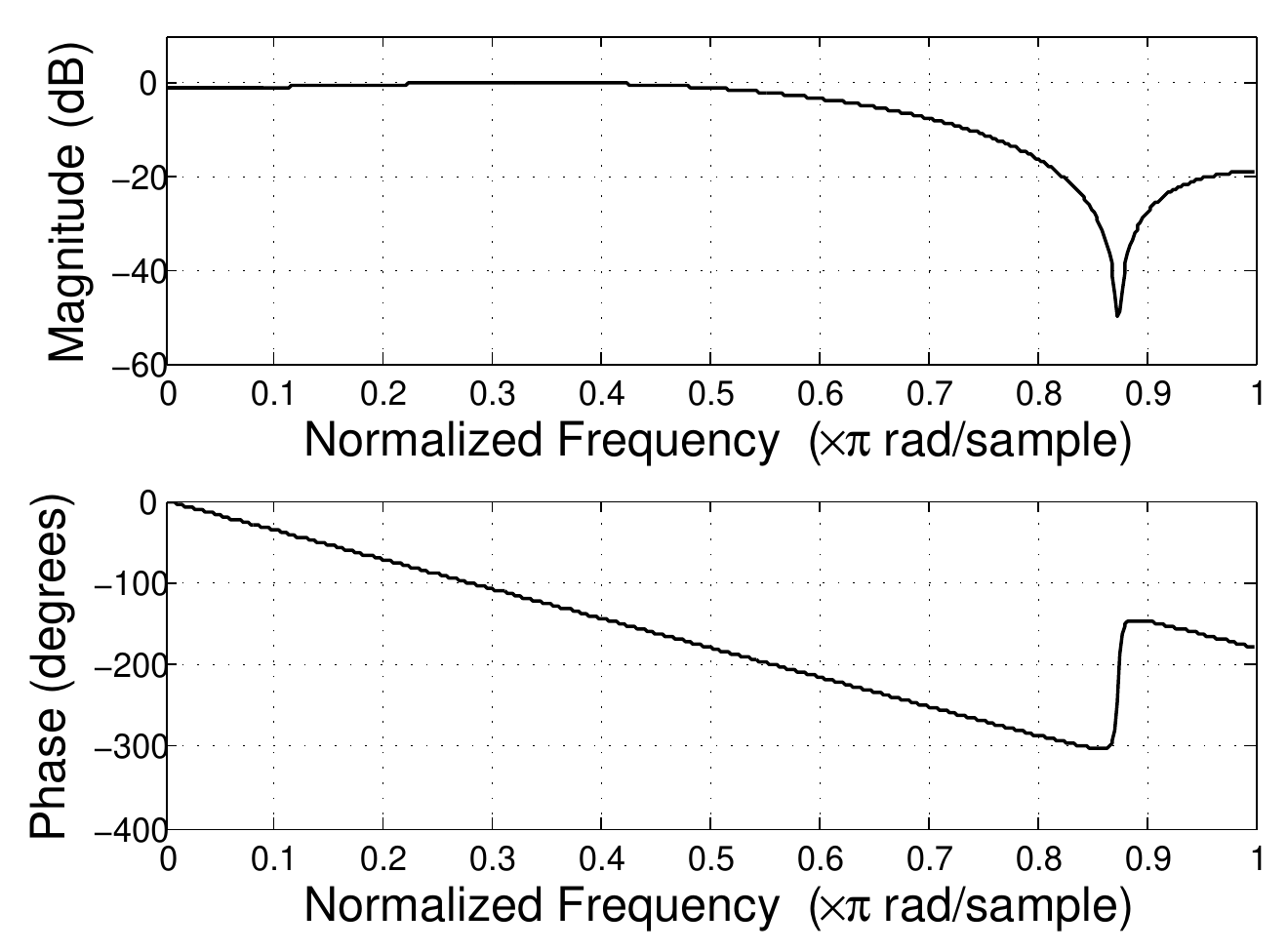}
\vspace{-0.6em}
\caption{Low pass filter}
\label{SubFigColumnLPF_WLP}
\end{subfigure}
\begin{subfigure}{0.48\textwidth}
\includegraphics[scale=0.58]{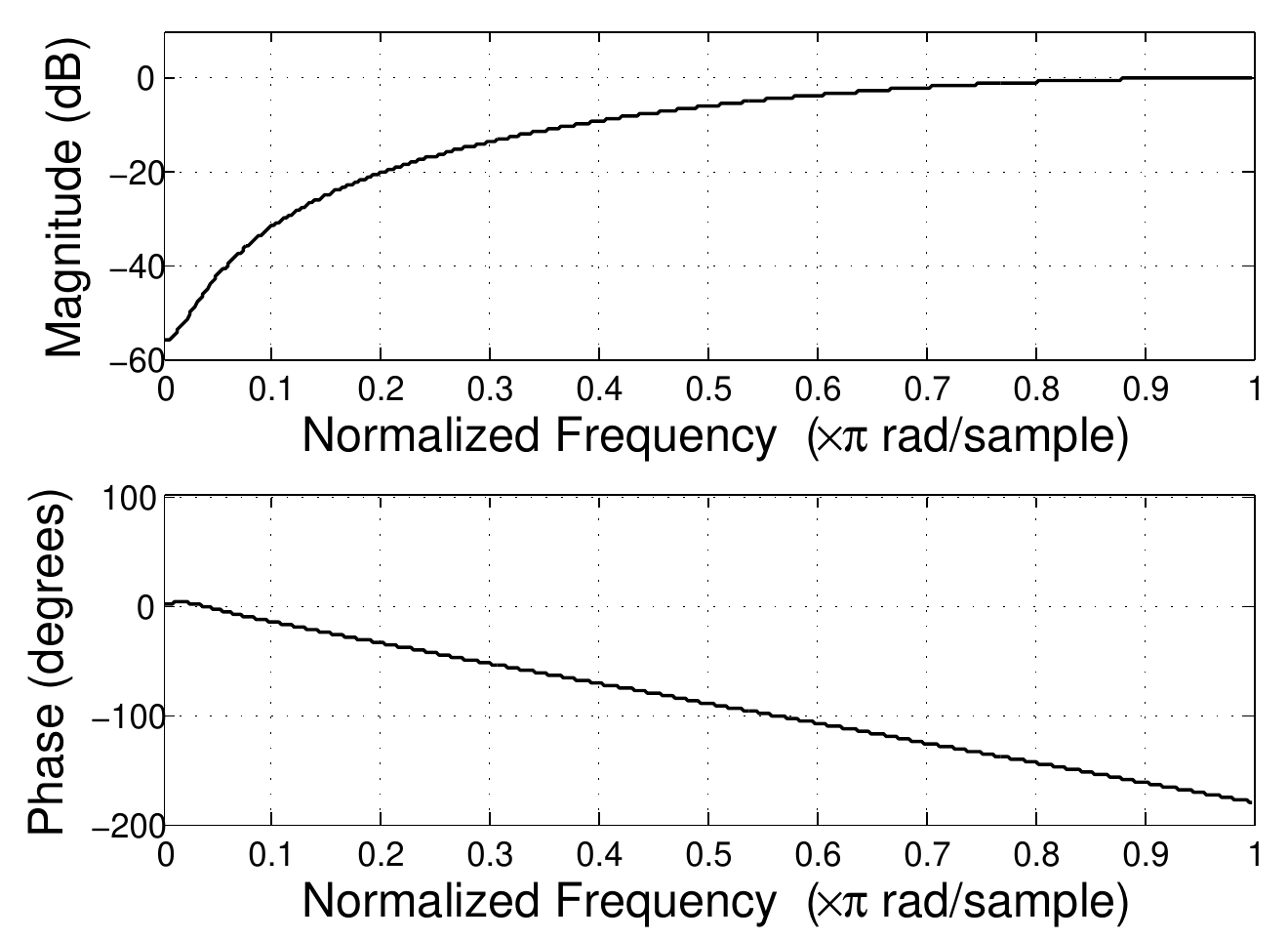} 
\vspace{-0.6em}
\caption{High pass filter}
\label{SubFigColumnHPF_WLP}
\end{subfigure}
\vspace{-0.6em}
\caption{Frequency response of 5/3 length analysis filters designed with `Lena' image sensed at 20\% sampling ratio}
\label{Fig for Frequency Response-1}
\end{figure*}
\vspace{-0.8em}
\subsection{Frequency Response and Filter Coefficients}
\label{SectionForFreqResponse}
We design image-matched wavelets for all the images shown in Fig. \ref{Fig for All Images}. Table-\ref{Table for Filter Coefficients} shows the analysis filter designed with the image `Lena' compressively sensed at a sampling ratio of 20\% with PCI sensing matrix. Coefficients for filters designed are shown in Table-\ref{Table for Filter Coefficients} for both column matched and row matched wavelets along with the coefficients of predict and update stage filters $T(z)$ and $S(z)$, respectively.  

\begin{table*}[!ht]
\vspace{-0.4em}
\centering
\caption{Coefficients of analysis filters designed with image `Lena'}
\vspace{-0.4em}
\label{Table for Filter Coefficients}
\begin{tabular}{ccc c} \hline
\multicolumn{4}{c}{Image matched 5/3 filters} \\ \hline
&\makecell{Predict Stage filter $T(z)$} & \makecell{Update Stage filter $S(z)$} & Filter coefficients \\ \hline
 \multirow{2}{*}{\makecell{Column  Matched}} &  \multirow{2}{*}{[0.5028 0.4941]} & \multirow{2}{*}{[0.2858 0.2790]}  & $h_0[n]$=[-0.1412 0.2858 0.7185 0.279 -0.1403] \\
&  & & $h_1[n]$=[-0.4941 1.0000 -0.5028] \\ \cline{1-4} 

 \multirow{2}{*}{\makecell{Row  Matched}}  & \multirow{2}{*}{[0.4959 0.5084]}  & \multirow{2}{*}{[0.2775 0.2871]} &$h_0[n]$=[-0.1411 0.2775 0.7164 0.2871 -0.1424] \\
& & & $h_1[n]$=[-0.5084 1.0000 -0.4959 ] \\ \cline{1-4} 
\end{tabular}
\vspace{-1em}
\end{table*}

\begin{figure*}[!ht]
\vspace{0em}
\centering
\begin{subfigure}{0.32\textwidth}
\centering
\includegraphics[scale=0.42]{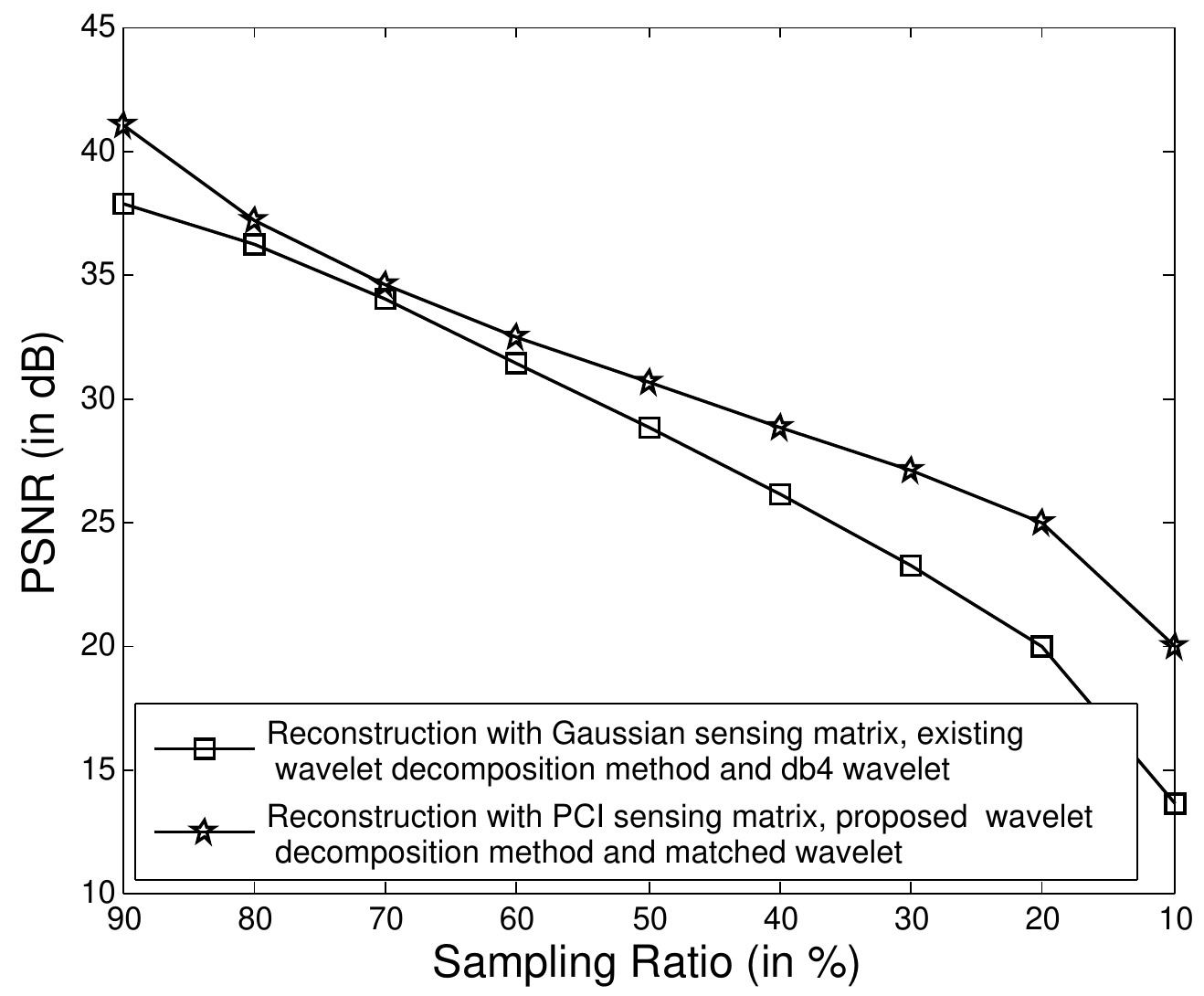}
\caption{Beads}
\end{subfigure}
\begin{subfigure}{0.32\textwidth}
\centering
\includegraphics[scale=0.42]{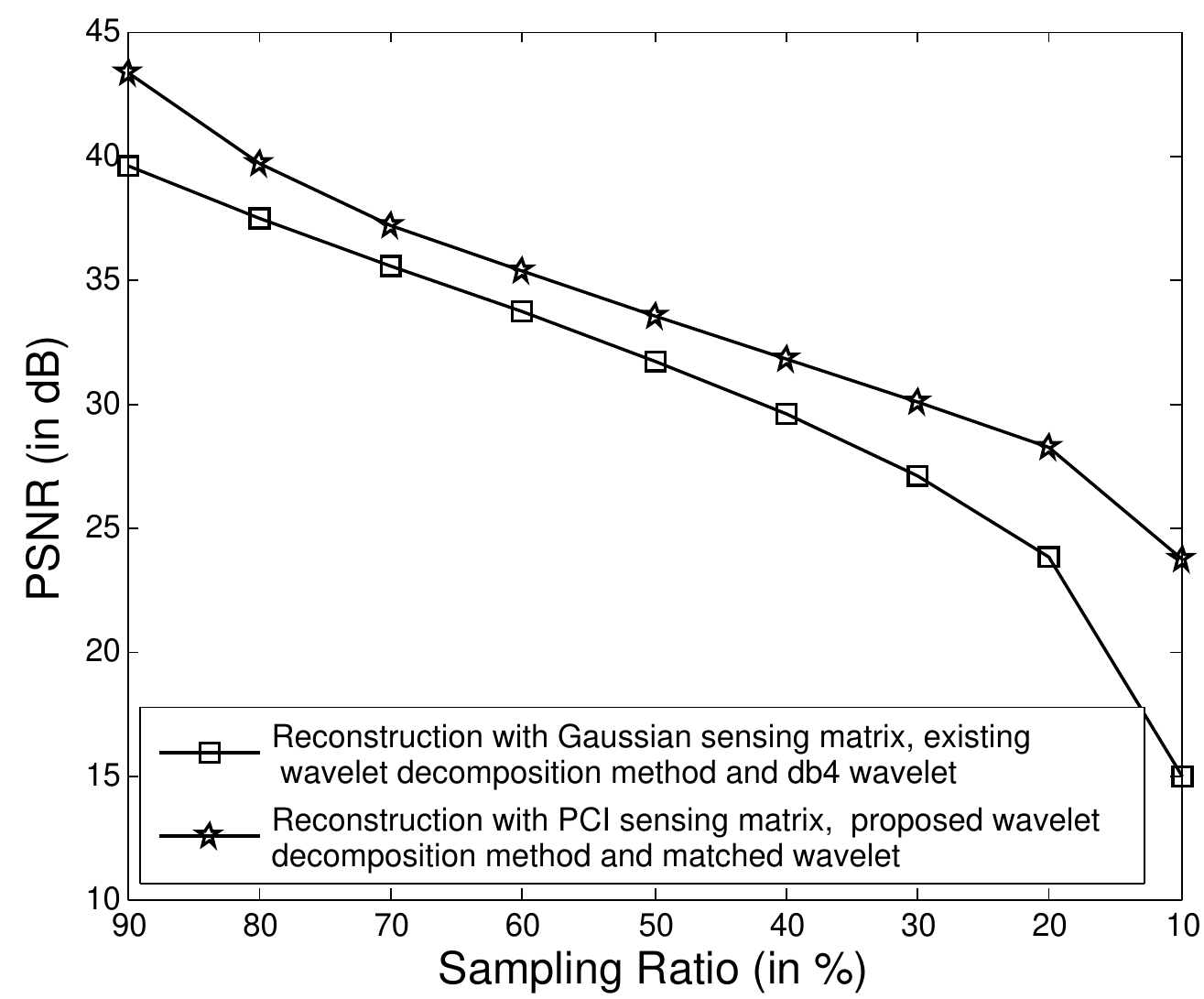}
\caption{Lena}
\end{subfigure}
\begin{subfigure}{0.32\textwidth}
\centering
\includegraphics[scale=0.42]{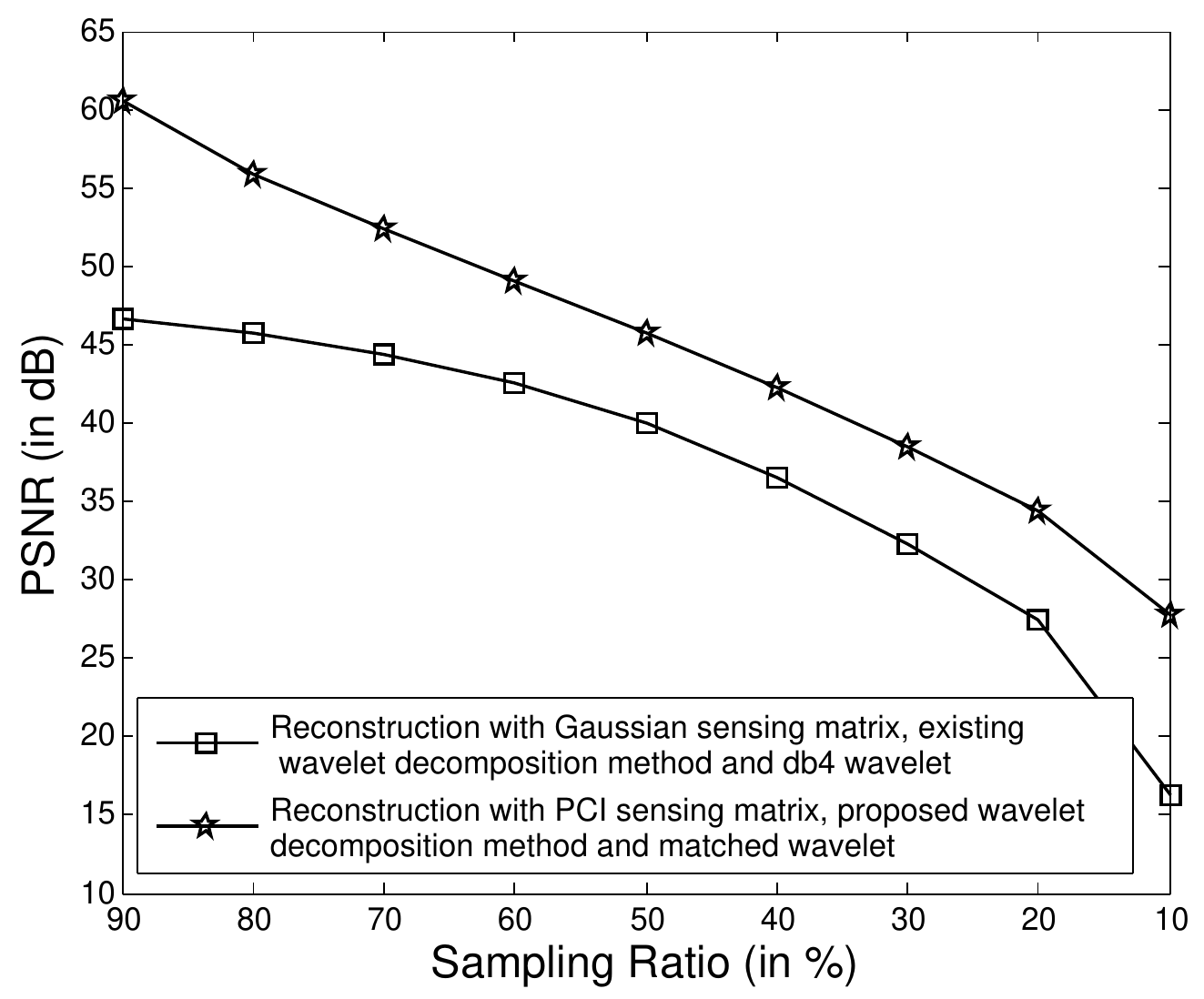}
\caption{House}
\end{subfigure}
\captionsetup{justification=centering}
\vspace{-0.6em}
\caption{\small CS-based reconstruction comparison of 1) the existing methodology: Gaussian sensing matrix, existing wavelet decomposition, and dB4 wavelet with 2) our proposed methodology: PCI sensing matrix, proposed wavelet decomposition, and proposed matched-wavelet design (5/3 length wavelet design).}
\vspace{-1.3em}
\label{FigForInitialVsFinalComparison}
\end{figure*}
\begin{figure*}[!ht]
\centering
\begin{subfigure}[b]{0.32\textwidth}
\includegraphics[scale=0.15]{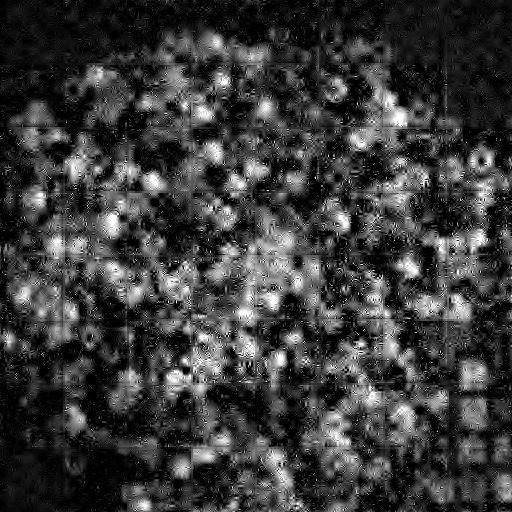}
\includegraphics[scale=0.15]{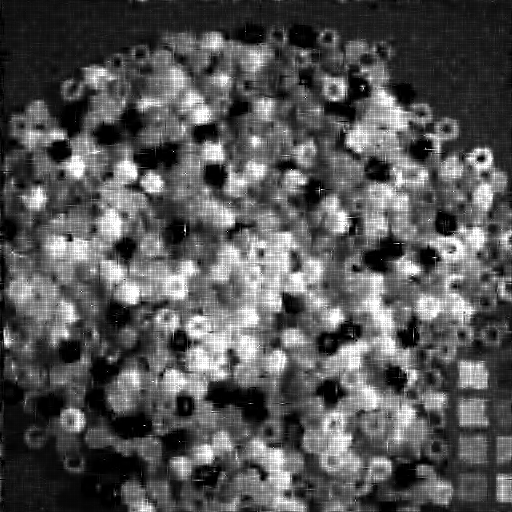}
\caption{Beads}
\vspace{-0.2em}
\label{Gaussian_Beads}
\end{subfigure}
\begin{subfigure}[b]{0.32\textwidth}
\includegraphics[scale=0.15]{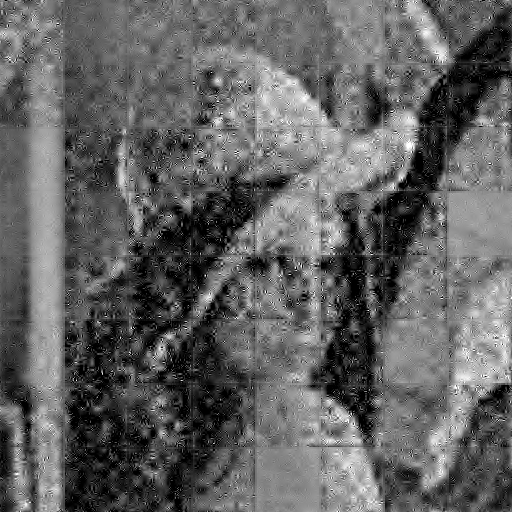}
\includegraphics[scale=0.15]{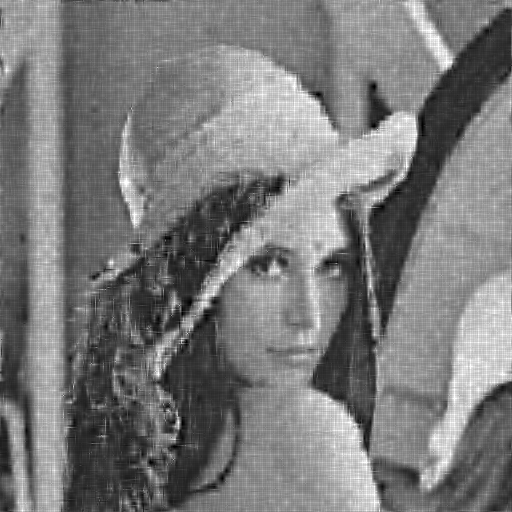}
\caption{Lena}
\vspace{-0.2em}
\label{Gaussian_Lena}
\end{subfigure}
\begin{subfigure}[b]{0.32\textwidth}
\includegraphics[scale=0.15]{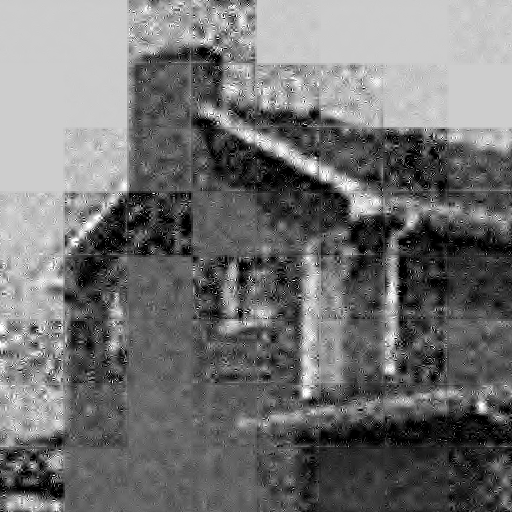}
\includegraphics[scale=0.15]{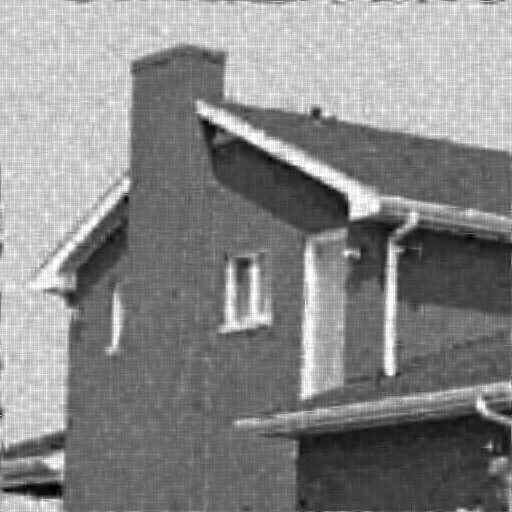}
\caption{House}
\vspace{-0.2em}
\label{Gaussian_House}
\end{subfigure}
\captionsetup{justification=centering}
\vspace{-0.6em}
\caption{\small Visual comparison of CS-based image reconstruction at 10\% sampling ratio with existing methodology and the proposed methodology. Left image in each subfigure shows image reconstructed with existing methodology while right image shows the image reconstructed with the proposed methodology.} 
\label{Fig for Gaussian reconstruction}
\vspace{-2em}
\end{figure*}
Fig. \ref{Fig for Frequency Response-1} shows normalized frequency responses of 5/3 length analysis filters designed for `Lena' image. These frequency response have been shown for only column-matched wavelets for brevity purposes.  

\textit{Discussion:} From the above figure, we note that the filters designed with the proposed method also exhibit good linearity in the phase response. This is because the proposed method of matched-wavelet design is itself able to capture linear phase characteristics of the image, if present.
\begin{table*}
\centering
\vspace{-0.5em}
\captionsetup{justification=centering}
\caption{Reconstruction accuracy on CS-based image reconstruction with standard wavelets and with image-matched wavelets designed from compressively sensed data with 5/3 length filters. PCI sensing matrix and the proposed wavelet decomposition strategy have been used to generate these results.}
\label{comparison with standard wavelets}
\begin{tabular}{ccccc ccccc ccc}
\hline
\makecell{Sampl-\\ing ratio} & \makecell{Wavelet \\used} & Beads  & Boat  & Lena  & \makecell{Ballo-\\ons}  & \makecell{Bar-\\bara}  & \makecell{Pepp-\\ers}  & \makecell{Mand-\\ril}  & House  & \makecell{Build-\\ing}  & \makecell{Camer-\\aman} & \makecell{Average \\ PSNR} \\ \hline
\multirow{5}{*}{90}             & db2          & 36.11 & 38.42 & 40.20 & 49.72 & 35.79 & 38.30 & 30.41 & 52.52 & 34.90 & 43.55 & 39.99   \\
               & db4          & 37.23 & 39.47 & 41.47 & 51.00 & 37.77 & 38.61 & 30.71 & 55.72 & 35.03 & 45.53 & 41.25   \\
               & bior 5/3      & 35.87 & 37.77 & 39.64 & 49.38 & 35.05 & 37.68 & 29.67 & 52.66 & 34.42 & 43.16 & 39.53   \\
               & Matched      & \textbf{41.02} & \textbf{41.07} & \textbf{43.35} & \textbf{52.41} & \textbf{36.40} & \textbf{38.02} & \textbf{31.63} & \textbf{60.52} & \textbf{37.30} & \textbf{50.54} & \textbf{43.23}   \\ \hline
                
\multirow{5}{*}{80}             & db2          & 32.35 & 34.55 & 36.47 & 46.30 & 31.82 & 35.02 & 27.09 & 47.55 & 31.43 & 38.96 & 36.15   \\
               & db4          & 33.51 & 35.54 & 37.69 & 47.19 & 33.52 & 35.27 & 27.35 & 50.77 & 31.54 & 40.91 & 37.33   \\
               & bior 5/3      & 32.10 & 33.93 & 35.84 & 45.67 & 31.12 & 34.40 & 26.35 & 48.02 & 30.96 & 38.67 & 35.71   \\
              & Matched      & \textbf{37.16} & \textbf{37.47} & \textbf{39.67} & \textbf{48.96} & \textbf{32.43} & \textbf{35.05} & \textbf{28.29} & \textbf{55.84} & \textbf{33.76} & \textbf{46.09} & \textbf{39.47}   \\ \hline
\multirow{5}{*}{70}             & db2          & 29.87 & 31.98 & 33.97 & 43.46 & 29.20 & 32.90 & 25.04 & 44.10 & 29.30 & 35.74 & 33.56   \\
               & db4          & 30.93 & 32.82 & 35.10 & 44.50 & 30.67 & 33.14 & 25.27 & 46.94 & 29.32 & 37.51 & 34.62   \\
               & bior 5/3      & 29.51 & 31.31 & 33.33 & 43.12 & 28.51 & 32.30 & 24.28 & 44.43 & 28.80 & 35.22 & 33.08   \\
               & Matched      & \textbf{34.57} & \textbf{35.08} & \textbf{37.20} & \textbf{46.52} & \textbf{29.71} & \textbf{33.27} & \textbf{26.23} & \textbf{52.37} & \textbf{31.51} & \textbf{42.53} & \textbf{36.90}   \\ \hline
\multirow{5}{*}{60}             & db2          & 27.83 & 29.96 & 31.94 & 41.19 & 27.16 & 31.17 & 23.50 & 40.98 & 27.53 & 33.22 & 31.45   \\
               & db4          & 28.65 & 30.65 & 32.98 & 42.10 & 28.36 & 31.41 & 23.68 & 43.26 & 27.48 & 34.55 & 32.31   \\
               & bior 5/3      & 27.36 & 29.23 & 31.26 & 40.93 & 26.49 & 30.61 & 22.77 & 41.11 & 27.03 & 32.39 & 30.92   \\
               & Matched      & \textbf{32.46} & \textbf{33.22} & \textbf{35.29} & \textbf{44.63} & \textbf{27.55} & \textbf{32.00} & \textbf{24.68} & \textbf{49.08} & \textbf{29.77} & \textbf{39.73} & \textbf{34.84}   \\ \hline
\multirow{5}{*}{50}             & db2          & 25.96 & 27.98 & 29.97 & 39.12 & 25.48 & 29.70 & 22.17 & 37.90 & 25.90 & 30.67 & 29.48   \\
               & db4          & 26.70 & 28.60 & 30.85 & 39.75 & 26.29 & 29.91 & 22.32 & 39.64 & 25.81 & 31.74 & 30.16   \\
               & bior 5/3      & 25.52 & 27.41 & 29.29 & 38.90 & 24.83 & 29.22 & 21.51 & 37.94 & 25.42 & 29.76 & 28.98   \\
               & Matched      & \textbf{30.59} & \textbf{31.52} & \textbf{33.47} & \textbf{42.73} & \textbf{25.86} & \textbf{30.94} & \textbf{23.38} & \textbf{45.73} & \textbf{28.34} & \textbf{37.00} & \textbf{32.96}   \\ \hline
\multirow{5}{*}{40}             & db2          & 24.14 & 26.07 & 28.00 & 36.48 & 23.87 & 28.03 & 20.95 & 34.71 & 24.15 & 28.14 & 27.45   \\
               & db4          & 24.75 & 26.51 & 28.75 & 36.64 & 24.45 & 28.25 & 21.04 & 35.57 & 24.04 & 28.98 & 27.90   \\
               & bior 5/3      & 23.75 & 25.57 & 27.50 & 36.57 & 23.36 & 27.71 & 20.37 & 34.72 & 23.77 & 27.17 & 27.05   \\
               & Matched      & \textbf{28.78} & \textbf{29.82} & \textbf{31.81} & \textbf{40.34} & \textbf{24.43} & \textbf{29.93} & \textbf{22.25} & \textbf{42.33} & \textbf{27.00} & \textbf{34.36} & \textbf{31.11}   \\ \hline
\multirow{5}{*}{30}             & db2          & 21.93 & 24.04 & 25.88 & 33.33 & 22.18 & 26.25 & 19.69 & 30.69 & 22.24 & 25.39 & 25.16   \\
               & db4          & 22.51 & 24.34 & 26.41 & 32.68 & 22.53 & 26.39 & 19.69 & 30.96 & 22.06 & 25.99 & 25.36   \\
               & bior 5/3      & 21.76 & 23.66 & 25.53 & 33.66 & 21.80 & 25.97 & 19.23 & 30.72 & 21.90 & 24.71 & 24.89   \\
               & Matched      & \textbf{27.05} & \textbf{28.22} & \textbf{30.06} & \textbf{37.06} & \textbf{23.21} & \textbf{28.93} & \textbf{21.18} & \textbf{38.50} & \textbf{25.64} & \textbf{31.75} & \textbf{29.16}   \\ \hline
\multirow{5}{*}{20}             & db2          & 18.84 & 21.07 & 22.85 & 28.74 & 19.66 & 23.00 & 17.81 & 25.36 & 19.56 & 21.84 & 21.87   \\
               & db4          & 19.59 & 21.42 & 23.27 & 28.01 & 20.03 & 23.60 & 17.89 & 25.83 & 19.61 & 22.32 & 22.16   \\
               & bior 5/3      & 18.82 & 20.81 & 22.65 & 29.01 & 19.43 & 23.03 & 17.51 & 25.37 & 19.27 & 21.46 & 21.74   \\
               & Matched      & \textbf{24.98} & \textbf{26.32} & \textbf{28.25} & \textbf{33.05} & \textbf{22.18} & \textbf{27.70} & \textbf{20.17} & \textbf{34.41} & \textbf{24.03} & \textbf{28.80} & \textbf{26.99}   \\ \hline
\multirow{5}{*}{10}             & db2          & 10.75 & 9.79  & 10.03 & 13.47 & 9.74  & 8.20  & 8.95  & 9.60  & 10.24 & 10.35 & 10.11   \\
               & db4          & 11.59 & 10.61 & 11.06 & 14.53 & 10.55 & 9.28  & 9.59  & 10.73 & 10.87 & 11.28 & 11.01   \\
               & bior 5/3      & 11.01 & 10.19 & 10.46 & 13.77 & 10.15 & 8.62  & 9.35  & 10.08 & 10.67 & 10.61 & 10.49   \\
               & Matched      & \textbf{20.00} & \textbf{22.29} & \textbf{23.74} & \textbf{30.95} & \textbf{20.08} & \textbf{23.04} & \textbf{18.08} & \textbf{27.69} & \textbf{20.26} & \textbf{23.03} & \textbf{22.92}  \\ \hline
\end{tabular}
\vspace{-1.4em}
\end{table*}
\vspace{-1.4em}

\subsection{Experiment-1: Comparison of existing CS-based image reconstruction methodology with the proposed methodology}
\vspace{-0.2em}
In the application of CS-based image reconstruction, the proposed methodology of this paper has three contributions:

\begin{enumerate}
\item \textit{Proposed use of PCI sensing matrix}: that is computationally very inexpensive compared to the existing Gaussian matrix (Fig. \ref{TimeComparisonWithMeasurementMatrices}), but provides approx. 2 dB lower performance compared to the existing Gaussian matrix (Fig. \ref{PSNRComparisonWithMeasurementMatrices}).
\item \textit{New wavelet decomposition strategy} (Fig. \ref{Fig forWavelet transform}): that provides better results in CS-based image reconstruction compared to the existing wavelet decomposition strategy (Fig. \ref{Fig for strategy comparison}). 
\item \textit{Design of image-matched wavelets}: wherein wavelets have been designed from compressively sensed images and used for images reconstruction. 
\end{enumerate}

Based on the above observations, we would like to compare the performance of the proposed CS-based image reconstruction (with all three novelties: Proposed use of PCI sensing matrix, proposed wavelet decomposition, and proposed matched 5/3 length wavelet design) with the existing CS-based reconstruction (Gaussian sensing matrix, existing wavelet decomposition, and dB4 wavelet). Fig. \ref{FigForInitialVsFinalComparison} shows CS-based reconstruction results in terms of PSNR averaged over 10 iterations on three images `Beads', `Lena', and `House'.

\textit{Discussion:} From Fig. \ref{FigForInitialVsFinalComparison} we observe that the proposed methodology performs consistently better at all sampling ratios with the following observations:
\begin{enumerate}
\item As the image changes from being rich in high frequency (`Beads') to mid-frequency (`Lena') to low frequency (`House') content, better and better reconstruction performance is observed. This is owing to the fact that the proposed matched wavelet design works best for signals rich in low frequency contents.
\item At the higher sampling ratio of 90\%, performance gain over standard methodology is 3 dB with `Beads', 4 dB with `Lena', and 13 dB with `House'. Since at higher sampling ratios, most of the input image samples are available upfront, hence, matched wavelet design is optimum. This provides very good performance and huge improvement over the existing methodology, particularly, for images rich in low frequency content.
\item At the lower sampling ratio of 10\%, performance gain over standard methodology is 6.5 dB with `Beads', 9 dB with `Lena', and 11 dB with `House'. In fact, standard methodology with standard wavelets almost fails in reconstructing images with any good quality at very low sampling ratios, while the proposed methodology still performs good. For visual clarity, Fig. \ref{Fig for Gaussian reconstruction} shows reconstructed images with the existing and the proposed methodology at 10\% sampling ratio. 
\end{enumerate}
\begin{figure*}[!ht]
\centering
\begin{subfigure}{0.98\textwidth}
\begin{center}
\includegraphics[scale=0.19]{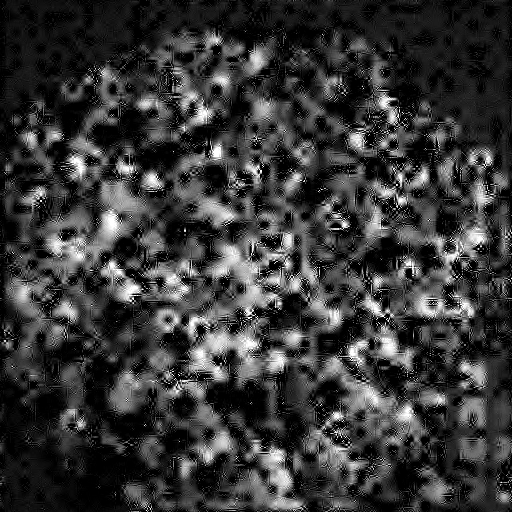}
\includegraphics[scale=0.19]{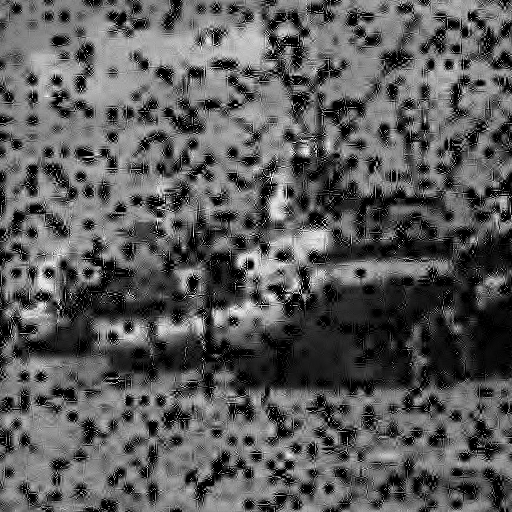}
\includegraphics[scale=0.19]{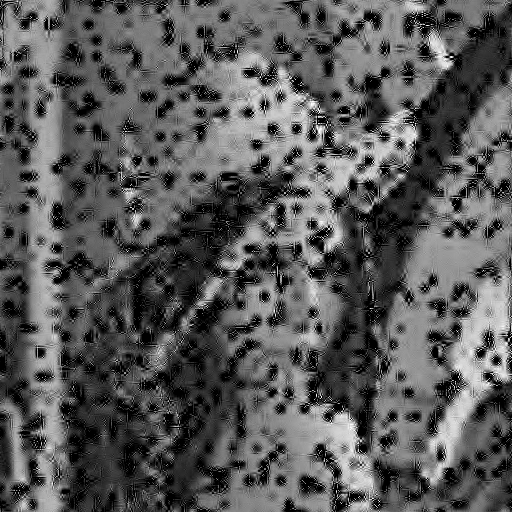}
\includegraphics[scale=0.19]{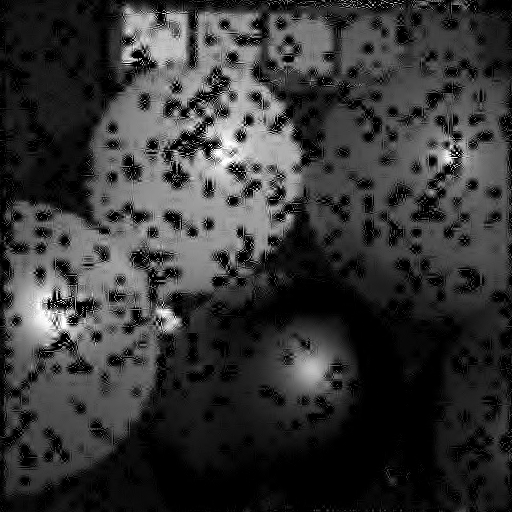}
\includegraphics[scale=0.19]{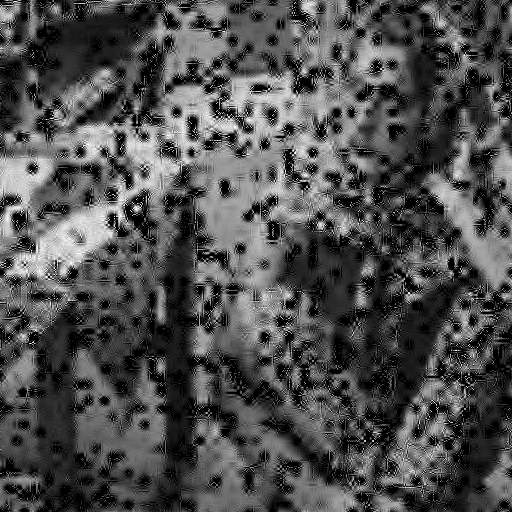}
\caption{Reconstructed with db4 wavelet}
\end{center}
\end{subfigure}
\begin{subfigure}{0.98\textwidth}
\begin{center}
\includegraphics[scale=0.19]{Beads_10_matched.png}
\includegraphics[scale=0.19]{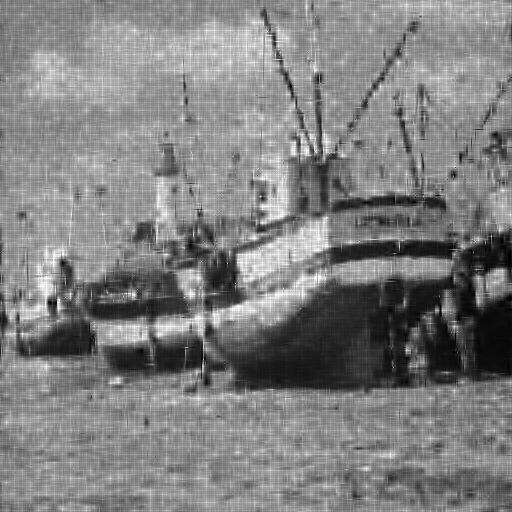}
\includegraphics[scale=0.19]{Lena_10_matched.png}
\includegraphics[scale=0.19]{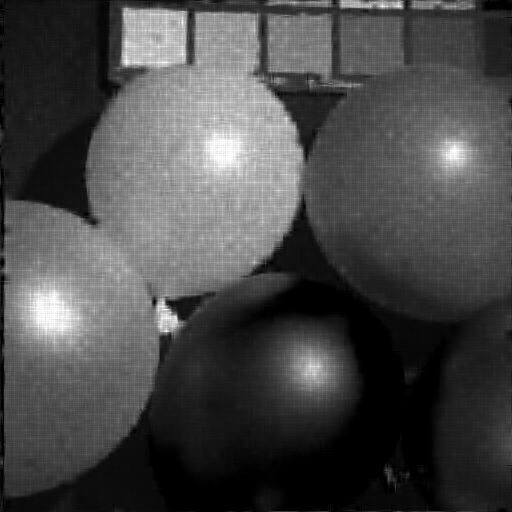}
\includegraphics[scale=0.19]{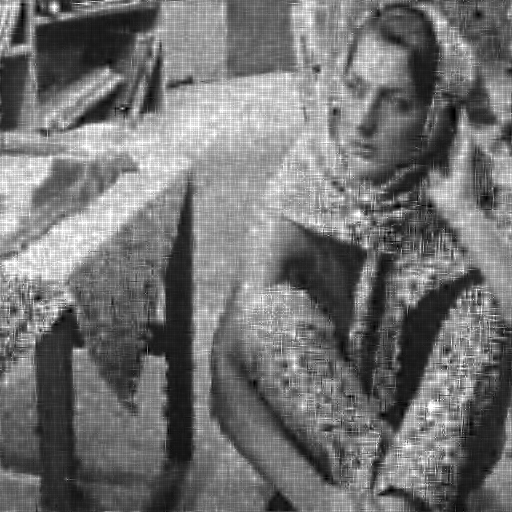}
\caption{Reconstructed with matched wavelet}
\end{center}
\end{subfigure}
\begin{subfigure}{0.98\textwidth}
\begin{center}
\includegraphics[scale=0.19]{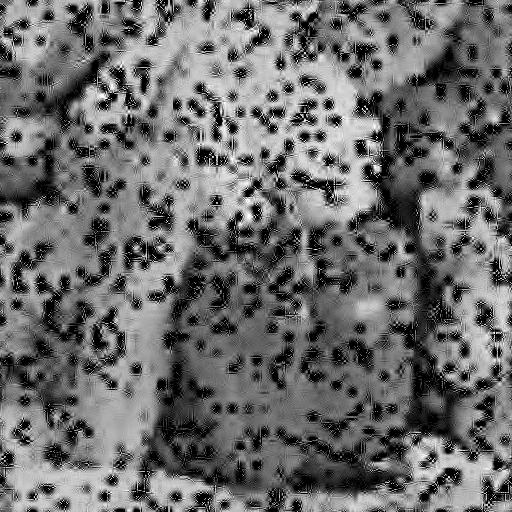}
\includegraphics[scale=0.19]{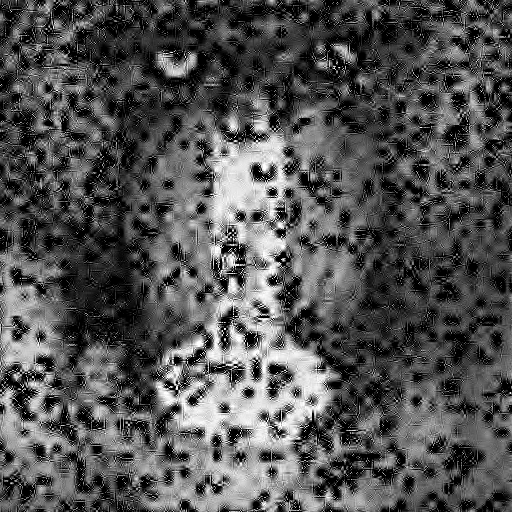}
\includegraphics[scale=0.19]{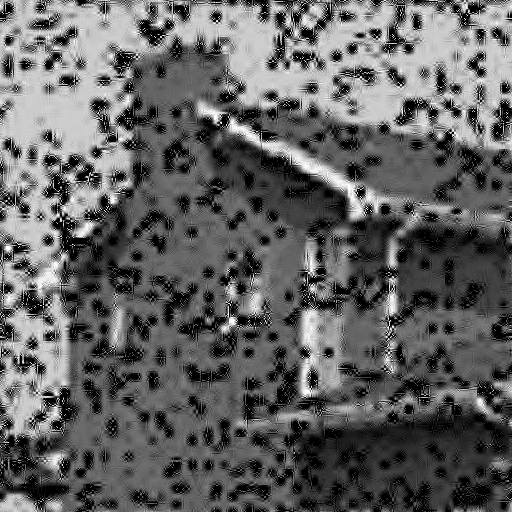}
\includegraphics[scale=0.19]{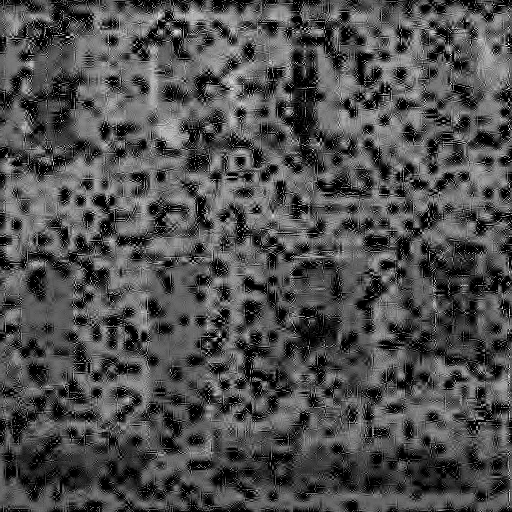}
\includegraphics[scale=0.19]{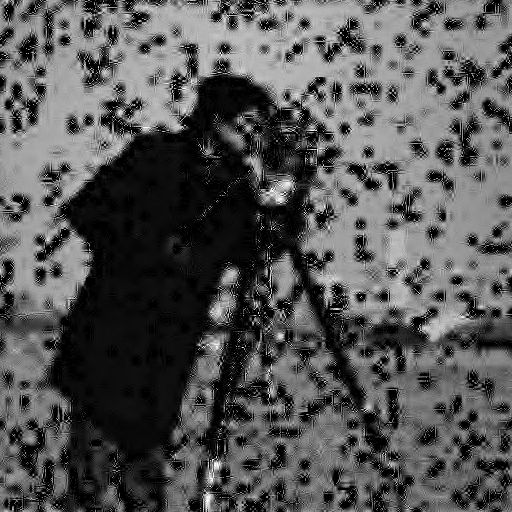}
\caption{Reconstructed with db4 wavelet}
\end{center}
\end{subfigure}
\begin{subfigure}{0.98\textwidth}
\begin{center}
\includegraphics[scale=0.19]{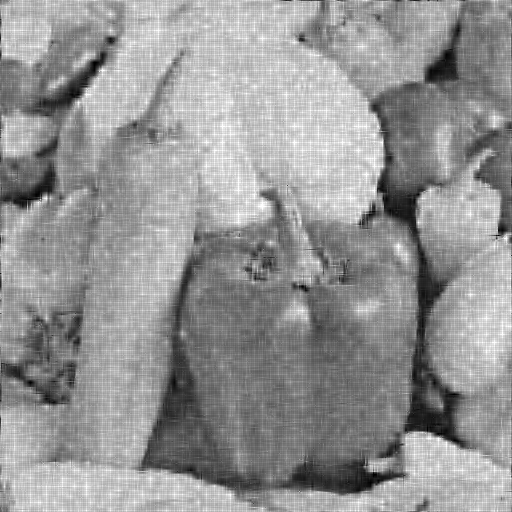}
\includegraphics[scale=0.19]{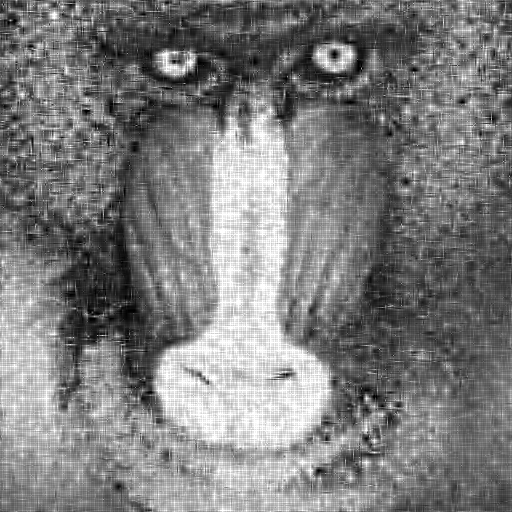}
\includegraphics[scale=0.19]{House_10_matched.png}
\includegraphics[scale=0.19]{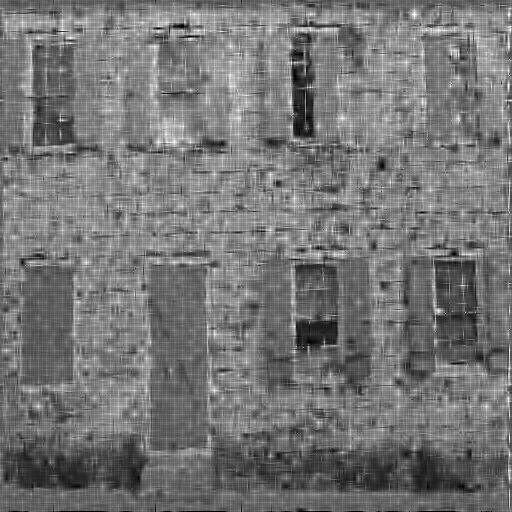}
\includegraphics[scale=0.19]{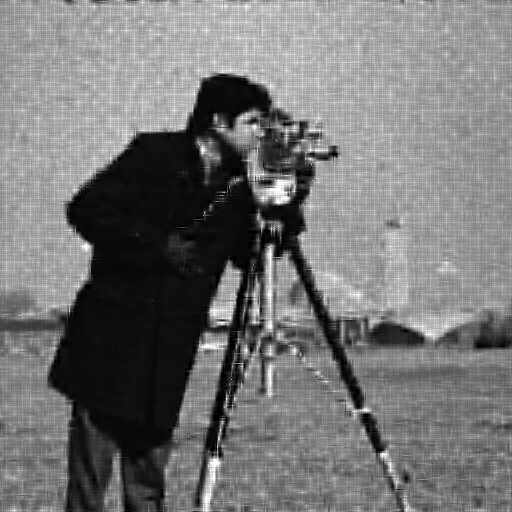}
\caption{Reconstructed with matched wavelet}
\end{center}
\end{subfigure}
\vspace{-0.8em}
\captionsetup{justification=centering}
\caption{\small Comparison of standard versus matched wavelet in CS-based reconstruction on images sensed at 10\% sampling ratio using PCI sensing matrix and the proposed wavelet decomposition strategy}
\label{Image reconstruction comparison}
\vspace{-2em}
\end{figure*}
\vspace{-1em}
\subsection{Experiment-2: Comparison of standard wavelets with matched wavelets in CS-based image reconstruction}
\vspace{-0.2em}
In this subsection, we present CS-based image reconstruction results using the PCI sensing matrix and the proposed wavelet decomposition strategy. Table-\ref{comparison with standard wavelets} shows comparison on reconstruction accuracy in terms of PSNR averaged over 10 iterations obtained with standard wavelets and that obtained with image-matched wavelets designed from compressively sensed images with 5/3 length filters.

We have considered three standard wavelets: orthogonal Daubechies' wavelets `db2', `db4', and biorthogonal standard 5/3 wavelet. While orthogonal wavelets are widely used in applications, biorthogonal wavelets take care of boundary effects and provide better compression results as compared to orthogonal wavelet \cite{usevitch2001tutorial}, \cite{lightstone1997low}. This is to note that our image-matched wavelets are biorthogonal by design. 
\vspace{-0.2em}

\textit{Discussion:} From Table-\ref{comparison with standard wavelets}, we observe much better reconstruction accuracy with matched wavelets compared to standard wavelets in terms of average PSNR over all the sampling ratio ranging from 90\% to 10\%. Orthogonal db4 wavelet performs better than other standard wavelets but its performance is still inferior to our matched wavelets. At the sampling ratio of 90\%, matched wavelet provides an improvement upto 3 dB, while we observe an improvement upto 11.5dB at 10\% sampling ratio. 
\vspace{-0.2em}

Further, we observe that while standard wavelets fail almost completely at lower sampling ratio of 10\% with a reconstruction PSNR of approximately 10dB only, matched wavelets are able to reconstruct images with a PSNR above 20dB. 

For better visual clarity, we have also shown images reconstructed from compressively sensed images at 10\% sampling ratio using the PCI matrix with matched 5/3 wavelet and standard db4 wavelet for all ten images in Fig. \ref{Image reconstruction comparison}. From the figure, it can be clearly noticed that the existing wavelet db4 is not able to reconstruct full images whereas matched wavelets provide good reconstruction quality.  
\vspace{-0.8em}
\section{Conclusion}
\label{Section for Conclusion}
\vspace{-0.2em}
In this paper, we have proposed a joint framework wherein image-matched wavelets have been designed from compressively sensed images and later, used for reconstruction or recovery of the full image. We have also proposed to use a partial canonical identity sensing matrix for CS-based reconstruction of images that performs much faster compared to the existing Gaussian or Bernoulli matrices and hence, is suited for time-bound reconstruction based applications. Although there is a slight degradation in performance with the proposed sensing matrix but that is easily covered up by the matched wavelet design. We have also provided a new multi-level L-Pyramid wavelet decomposition strategy that works much more efficiently compared to the standard wavelet decomposition method. Overall, the proposed work with different sensing matrix, new wavelet decomposition strategy, and image-matched wavelets provide much better results in CS-based image reconstruction compared to the existing practices.
\vspace{-1.5em}
\ifCLASSOPTIONcaptionsoff
  \newpage
\fi
\bibliographystyle{IEEEtran}
\bibliography{refs}




\end{document}